%% file: camera-ready-1.tex
\documentclass{article} 
\usepackage{iclr2025_conference,times}

\input{math_commands.tex}

\usepackage{hyperref}
\usepackage{url}

\usepackage{calc}
\usepackage{amsthm}
\usepackage{booktabs}
\usepackage{adjustbox}
\usepackage{subfig}
\usepackage{wrapfig}
\usepackage{subcaption}
\usepackage{aligned-overset}
\usepackage{thmtools,thm-restate}
\newcommand{\Matern}{Mat\'{e}rn }

\newtheorem{theorem}{Theorem}

\newtheorem{remark}[theorem]{Remark}
\DeclareMathOperator*{\minimize}{minimize}
\usepackage{multirow}

\title{\Matern Kernels for Tunable Implicit Surface Reconstruction}

\author{Maximilian Weiherer \& Bernhard Egger \\
Department of Computer Science \\
Friedrich-Alexander-Universität Erlangen-Nürnberg \\
\texttt{\{maximilian.weiherer,bernhard.egger\}@fau.de}
}

\iclrfinalcopy 
\begin{document}

\maketitle

\begin{abstract}
We propose to use the family of \Matern kernels for implicit surface reconstruction, building upon the recent success of kernel methods for 3D reconstruction of oriented point clouds.
As we show from a theoretical and practical perspective, \Matern kernels have some appealing properties which make them particularly well suited for surface reconstruction---outperforming state-of-the-art methods based on the \textit{arc-cosine} kernel while being significantly easier to implement, faster to compute, and scalable.
Being stationary, we demonstrate that \Matern kernels allow for \textit{tunable} surface reconstruction in the same way as Fourier feature mappings help coordinate-based MLPs overcome spectral bias. 
Moreover, we theoretically analyze \Matern kernels' connection to \textsc{Siren} networks as well as their relation to previously employed arc-cosine kernels. 
Finally, based on recently introduced Neural Kernel Fields, we present \textit{data-dependent} \Matern kernels and conclude that especially the Laplace kernel (being part of the \Matern family) is extremely competitive, performing almost on par with state-of-the-art methods in the noise-free case while having a more than five times shorter training time.\footnote{Code available at: \url{https://github.com/mweiherer/matern-surface-reconstruction}.}
\end{abstract}

\section{Introduction}
Recovering the shape of objects from sparse or only partial observations is a challenging task.
Formally, let $\Omega\subset\R^d$ and $\mathcal{X}=\{x_1,x_2,\dots,x_m\}\subset\Omega$ be a set of $m$ points in $d$ dimensions which forms, together with associated per-point normals, a dataset $\mathcal{D}=\mathcal{X}\times\{n_1,n_2,\dots,n_m\}\subset\Omega\times\R^d$.
The goal of surface reconstruction is to recover the object's shape from $\mathcal{D}$.
The shape of objects may be represented as dense point clouds, polygonal meshes, manifold atlases, voxel grids, or (as the zero-level set of) implicit functions, which is the representation we chose to focus on in this work.
Specifically, if $f:\mathbb{R}^d\longrightarrow\mathbb{R}$ denotes an implicit function, such as a \textit{Signed Distance Function} (SDF), then, from a practical perspective, implicit surface reconstruction aims at finding an optimal solution $\hat{f}$ to the following \textit{Kernel Ridge Regression} (KRR) problem:
\begin{equation}
    \hat{f}=\argmin_{f\in H}\left\{\sum_{i=1}^m|f(x_i)|^2+\Vert\nabla f(x_i)-n_i \Vert^2+\lambda\Vert f\Vert_H^2\right\},
    \label{eq:emp_risk}
\end{equation}
where the function space $H:=H(\Omega)$ is usually taken to be a \textit{Reproducing Kernel Hilbert Space} (RKHS) with associated reproducing kernel $k:\Omega\times\Omega\longrightarrow\mathbb{R}$, and $\lambda>0$.
Once $\hat{f}$ at hand, the reconstructed object's surface is implicitly given as $\mathcal{\hat{S}}=\{x:\hat{f}(x)=0\}\subset\mathbb{R}^d$ and can be extracted using Marching Cubes~\citep{lorensen1987}.
It is quite easy to observe that, as $\lambda\rightarrow 0$, the optimization problem in Eq. (\ref{eq:emp_risk}) turns into the constrained minimization problem,
\begin{equation}
    \minimize_{f\in H}\Vert f\Vert_H\quad\text{s.t.}\quad f(x_i)=0\quad\text{and}\quad\nabla f(x_i)=n_i\quad\forall i\in\{1,2,\dots,m\}
    \label{eq:constr_min}
\end{equation}
which makes it easier to see how the predicted surface behaves \textit{away} from the input points.
As seen from Eq. (\ref{eq:constr_min}), while $\mathcal{\hat{S}}$ should interpolate the given points $\mathcal{X}$ exactly, the behavior in between and away from those points is solely determined by the induced norm $\Vert f\Vert_H$ of the function space, $H$, over which we are optimizing.
The properties of $\Vert f\Vert_H$ are uniquely defined by the reproducing kernel, $k$; this is a direct consequence of the Representer Theorem \citep{kimeldorf1970,schoelkopf2001} which states that solutions to Eq. (\ref{eq:emp_risk}) are of the form $f(x)=\sum_{i=1}^m\alpha_i k(x,x_i)$, yielding $\Vert f\Vert_H^2=\alpha^\top K\alpha$.
Here, $\alpha\in\R^m$ and $K=(K)_{ij}=k(x_i,x_j)\in\R^{m\times m}$.
This shows that the behavior of the norm and hence, estimated surfaces $\mathcal{\hat{S}}$, can be explicitly controlled by the kernel function, enabling easy injection of \textit{inductive biases} (such as smoothness assumptions) into the surface reconstruction problem.
\begin{figure}
    \centering
    \subfloat{\includegraphics[width=0.48\textwidth]{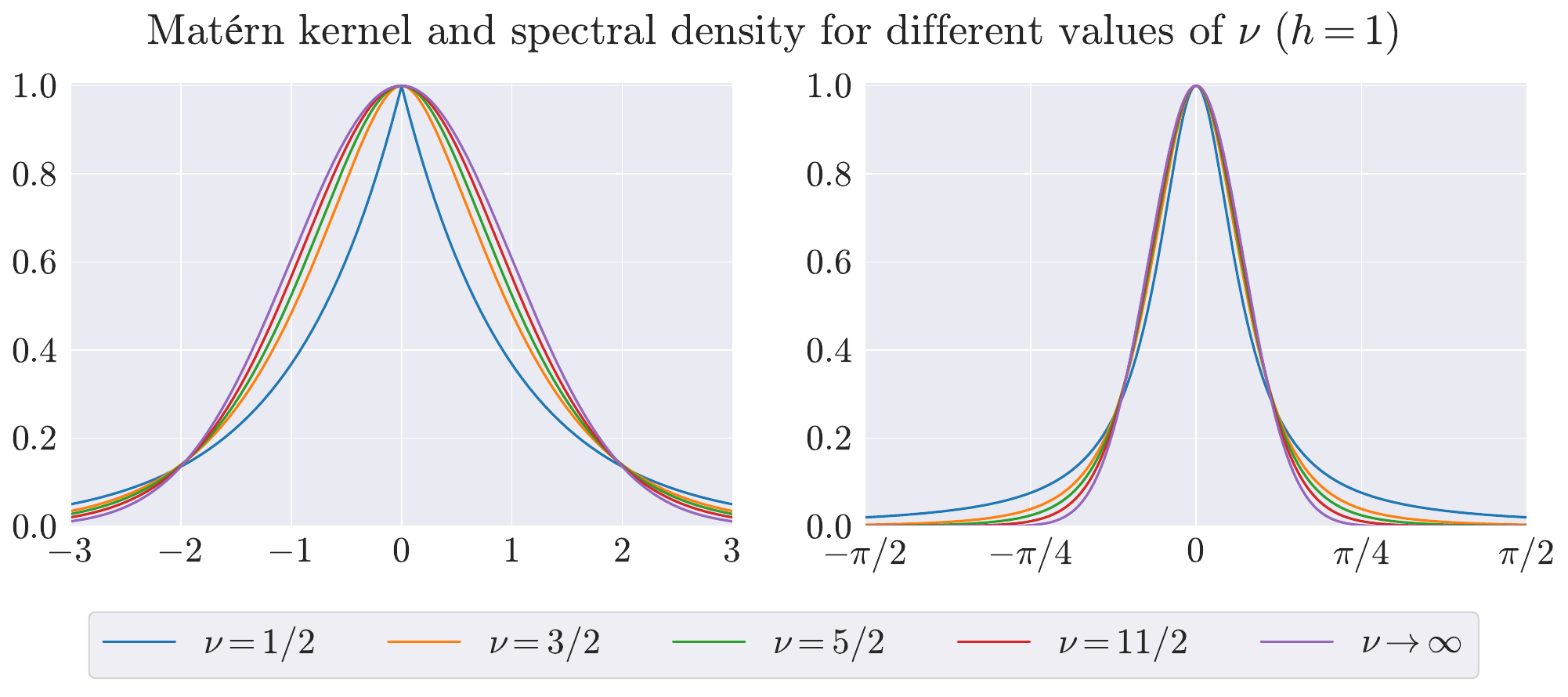}}
    \hfill
    \subfloat{\includegraphics[width=0.48\textwidth]{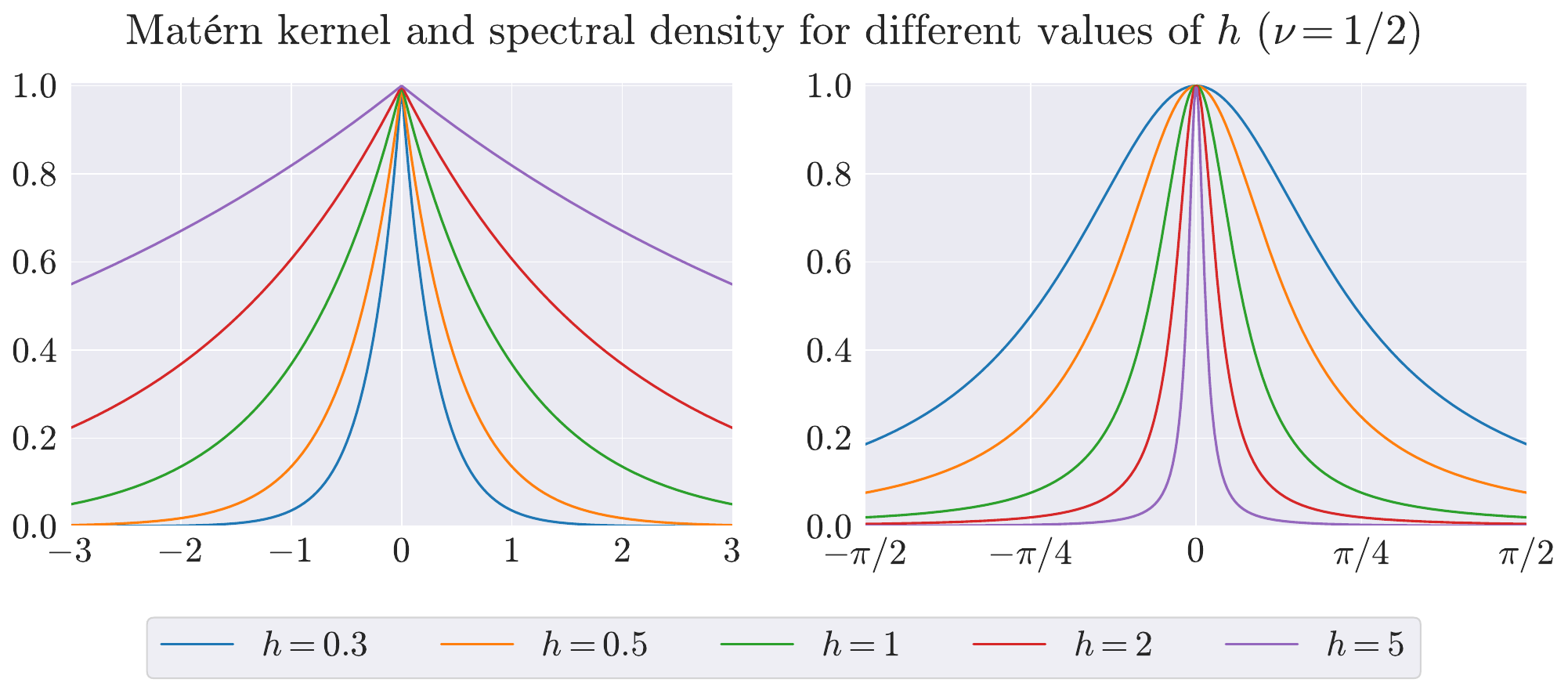}}
    \caption{\textbf{\Matern kernels with associated spectral densities.} We propose to use the family of \Matern kernels for tunable implicit surface reconstruction, parametrized by a smoothness parameter, $\nu>0$, that controls the differentiability of the kernel, and a bandwidth parameter $h>0$. Both parameters allow explicit manipulation of the kernel and its spectrum. Importantly, \Matern kernels recover the Laplace kernel for $\nu=1/2$ and the Gaussian kernel as $\nu\rightarrow\infty$.}
    \label{fig:matern_kernels}
\end{figure}
Even further, if the chosen kernel has adjustable parameters, they can be used to \textit{adaptively} (on a shape-by-shape basis) manipulate the inductive bias.

Although already used in the mid-90s and early 2000s \citep{savchenko1995,car1997,car2001}, only recently, kernel-based methods for 3D implicit surface reconstruction became extremely competitive, with \textit{Neural Kernel Surface Reconstruction} (NKSR; \cite{huang2023}) eventually evolving as the new state-of-the-art.
While early works mostly focused on (polyharmonic) radial basis functions such as thin-plate splines, recent work \citep{williams2021,williams2022} employ the first-order \textit{arc-cosine} kernel.
Introduced by \cite{cho2009}, arc-cosine kernels have been shown to mimic the computation of two-layer fully-connected ReLU networks; indeed, in the infinite-width limit and with \textit{fixed} bottom layer weights, the first-order arc-cosine kernel becomes the \textit{Neural Tangent Kernel} (NTK; \cite{jacot2018}) of the network.
One important aspect of this kernel is that it does not have adjustable parameters.
While this can be advantageous in some situations, it also prevents users from \textit{tuning} their reconstructions.
If the arc-cosine kernel fails to accurately recover a surface, then there is no possibility (except for using more observations) to improve the result.

\begin{wrapfigure}[22]{r}{0.48\textwidth}
    \vspace{-0.35cm}
    \includegraphics[width=0.48\textwidth]{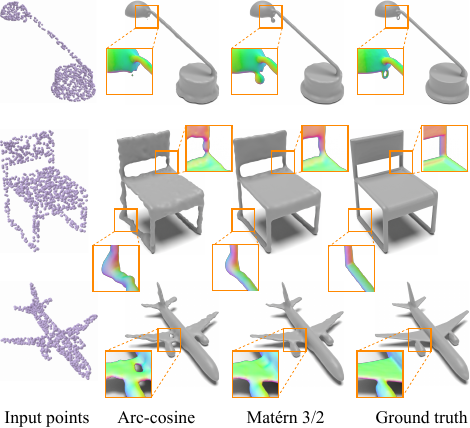}
    \caption{\textbf{Tunable \Matern kernels lead to better surface reconstructions than the previously employed arc-cosine kernel}. Here, we show surface reconstructions from sparse point clouds of just 1,000 points.}
    \label{fig:recons_ns_matern}
\end{wrapfigure}
In this work, we suggest a different family of kernels (with parameters $\nu,h>0$) for implicit surface reconstruction: \Matern kernels \citep{matern1986,stein1999}, see Figures \ref{fig:matern_kernels} and \ref{fig:recons_ns_matern}.
Contrary to arc-cosine kernels, \Matern kernels are stationary (hence translation invariant), spatially decaying (thus leading to sparse Gram matrices), and unify a variety of popular kernel functions, including the Laplace and Gaussian kernel.
As we will show, \Matern kernels have appealing properties, making them the ideal candidate for kernel-based surface reconstruction.
From a practical perspective, we demonstrate that a simple change of the kernel function from arc-cosine to \Matern leads to a consistently improved reconstruction accuracy and a \textit{significant} speedup.
From a theoretical perspective, we argue that \Matern kernels allow for \textit{tunable} surface reconstruction in the same way as Fourier feature mappings \citep{tancik2020} enable coordinate-based \textit{Multi-layer Perceptrons} (MLPs) to learn high-frequency functions in low-dimensional domains, effectively overcoming \textit{spectral bias} (the slow convergence towards high-frequency details).
In addition, we prove that, as the layer-width approaches infinity, with fixed bottom layer weights and under certain initializations, \Matern kernels are identified with the NTK of two-layer \textsc{Siren} (fully-connected MLPs with sine activations; \cite{sitzmann2020}) networks---together with Fourier feature mappings the two arguably most influential methods to overcome spectral bias in coordinate-based MLPs.
Lastly, we establish a connection between arc-cosine and \Matern kernels by showing that the Laplace (\Matern kernel with $\nu=1/2$) and first-order arc-cosine kernel are equal up to scaling and addition of another kernel function when restricted to the sphere.
In summary, our key contributions are:
\begin{itemize}
    \item We propose to use \Matern kernels for \textit{tunable} implicit surface reconstruction.
    \item We theoretically analyze \Matern kernels, relating them to Fourier feature mappings, \textsc{Siren} networks, and arc-cosine kernels. Moreover, we derive practical insights into how to choose the tunable bandwidth parameter based on a new bound of the $L_2$ reconstruction error.
    \item We propose data-dependent (\textit{i.e.}, learnable) \Matern kernels by leveraging the \textit{Neural Kernel Field} (NKF) framework \citep{williams2022}.
\end{itemize}
Our experimental evaluation reveals that \Matern 1/2 and 3/2 are extremely competitive, outperforming the arc-cosine kernel while being significantly easier to implement (essentially two lines of standard PyTorch code), faster to compute, and scalable.
In addition to geometry, we show that \Matern kernels surpass the arc-cosine kernel in reconstructing other high-frequency scene attributes, such as texture. 
Finally, we demonstrate that learnable \Matern kernels (1) outperform the data-dependent arc-cosine kernel (as implemented in the original NKF framework) while being more than four times faster to train, and (2) perform almost on par with highly sophisticated and well-engineered NKSR in the noise-free case while having a more than five times shorter training time. 

\section{Related Work}
We briefly review some relevant literature about 3D \textit{implicit} surface reconstruction from oriented point clouds, focusing on kernel-based methods. 
For a more in-depth overview, including neural-network-based reconstruction methods, see surveys by~\cite{berger2017,huang2022}.

Early kernel-based 3D surface reconstruction methods mostly employ \textit{Radial Basis Functions} (RBFs) such as thin-plate splines~\citep{savchenko1995}, biharmonic RBFs~\citep{car1997,car2001}, or Gaussian kernels~\citep{schoelkopf2005}.
Nowadays, the most widely used surface reconstruction technique is \textit{Screened Poisson Surface Reconstruction} (SPSR;~\cite{kazhdan2013}) which, however, can itself be viewed as kernel method~\citep{williams2021}.
Only recently, \textit{Neural Splines} (NS;~\cite{williams2021}) proposed to use the so-called (first-order) \textit{arc-cosine} kernel,
\begin{equation}
    k_\text{AC}(x,y)=\frac{\Vert x\Vert\Vert y\Vert}{\pi}\left(\sin\theta+(\pi-\theta)\cos\theta\right),\quad\text{where}\quad\theta=\cos^{-1}\left(\frac{x^\top y}{\Vert x\Vert\Vert y\Vert}\right)
    \label{eq:def_arc_cosine}
\end{equation}
for implicit surface reconstruction, achieving state-of-the-art results that outperform classical surface reconstruction techniques and non-linear methods based on neural networks by a large margin.
This method laid the cornerstone for \textit{Neural Kernel Fields} (NKFs;~\cite{williams2022}) which attempts to make arc-cosine kernels learnable by passing input points through a task-specific neural network before evaluating the kernel function, similar to Deep Kernel Learning~\citep{wilson2016}.
Based on SPSR, NeuralGalerkin~\citep{huang2022_ng} proposed to \textit{learn} basis functions inferred by a sparse convolutional network instead of using a fixed B\'{e}zier basis as in SPSR, hence can be seen as kernel method in the broader sense.
\textit{Neural Kernel Surface Reconstruction} (NKSR;~\cite{huang2023}) built upon NKF and NeuralGalerkin and proposed an all-purpose, highly scalable surface reconstruction method that is robust against noise and eventually became state-of-the-art.

\section{\Matern Kernels for Tunable Surface Reconstruction}
We propose to use the family of \textit{\Matern}kernels~\citep{matern1986,stein1999} for implicit surface reconstruction, being defined as
\begin{equation}
    k_\nu(x,y)=\Phi_\nu(\Vert x-y\Vert)=\Phi_\nu(\tau)=\frac{2^{1-\nu}}{\Gamma(\nu)}\left(\frac{\sqrt{2\nu}\tau}{h}\right)^\nu K_\nu\left(\frac{\sqrt{2\nu}\tau}{h}\right),
    \label{eq:def_matern}
\end{equation}
where $\nu>0$ is a \textit{smoothness parameter} that explicitly controls the differentiability, and $h>0$ is known as the shape parameter (or \textit{bandwidth}) of the kernel.
$\Gamma$ denotes the Gamma function, and $K_\nu$ is the modified Bessel function of the second kind of order $\nu$.
\Matern kernels generalize a variety of other kernel functions; the most popular ones can be written in closed form as
\begin{align}
    \nu&=1/2:\Phi_{\text{1/2}}(\tau)=\exp\left(-\frac{\tau}{h}\right),\\
    \nu&=3/2:\Phi_{\text{3/2}}(\tau)=\exp\left(-\frac{\sqrt{3}\tau}{h}\right)\left(1+\frac{\sqrt{3}\tau}{h}\right),\\
    \nu&=5/2:\Phi_{\text{5/2}}(\tau)=\exp\left(-\frac{\sqrt{5}\tau}{h}\right)\left(1+\frac{\sqrt{5}\tau}{h}+\frac{5\tau^2}{3h^2}\right),
\end{align}
where $\Phi_{\text{1/2}}$ is known as the Laplace kernel. 
In the limiting case of $\nu\rightarrow\infty$, \Matern kernels recover the popular Gaussian kernel:
\begin{equation}\label{eq:gaussian_kernel}
    \Phi_\infty(\tau):=\lim_{\nu\rightarrow\infty}\Phi_\nu(\tau)=\exp\left(-\frac{\tau^2}{2h^2}\right).
\end{equation}
For more information about \Matern kernels, please refer to \cite{porcu2023}.
Next, we revisit some basic properties of \Matern kernels and place them in the context of surface reconstruction.  

\subsection{Basic Properties}
\label{subsec:basic_props}
\textbf{Differentiable.}
\Matern kernels allow for \textit{controlled} smoothness, being exactly $\lceil\nu\rceil-1$ times differentiable (in the mean-square sense).
Since functions $f$ in an RKHS $H$ inherit the differentiability class of the inducing kernel $k$ due to the reproducing property, the from Eq. (\ref{eq:emp_risk}) reconstructed surface $\mathcal{\hat{S}}$ enjoys the same smoothness properties as $k$.
In the context of 3D surface reconstruction, this allows an easy injection of inductive biases into the reconstruction problem; for instance, if the roughness or noisiness of the objects to be reconstructed is known in advance, one can adjust the smoothness of the kernel accordingly.
This is not possible with the arc-cosine kernel.

\textbf{Stationary.}
\Matern kernels are stationary, \textit{i.e.}, they only depend on the difference $x-y$ between two points $x,y\in\Omega$.
If the distance is Euclidean, \Matern kernels are also \textit{isotropic}; they only depend on $\Vert x-y\Vert$.
As a result, \Matern kernels are rotation \textit{and} translation invariant, hence being independent of the absolute embedding of points $\mathcal{X}$ (based on the reproducing property, translating input points does not change the reconstructed surface).
Consequently, since the kernel's value only depends on the \textit{relative} positions, objects in a scene (or parts of an object) with similar geometric properties will be reconstructed consistently as long as the relative distances between points on each instance (or part) remain the same.
This is in contrast to the arc-cosine kernel, whose value depends on the \textit{absolute} position of points, see Eq. (\ref{eq:def_arc_cosine}); hence, it is \textit{not} translation invariant (thus \textit{non-stationary}), and multiple instances of an object in a scene will generally not yield the same surface reconstruction.
See Appendix \ref{ap:sec:translation_matern} for further discussion.

\textbf{Spatially decaying.}
As opposed to arc-cosine kernels, \Matern kernels are spatially decaying; they tend to zero as the distance $\Vert x-y\Vert$ becomes large. 
Consequently, although \Matern kernels are technically not compact (or locally supported), because their value decays exponentially fast \citep{porcu2023}, they can be considered \textit{effectively compact}.
From a computational perspective, this is a very desirable property as it leads to \textit{sparse} Gram matrices when truncated with a special kernel \citep{genton2001}, allowing the use of highly efficient and scalable sparse linear solvers.
Finally, we note that simply truncating any kernel does not yield a valid, positive definite kernel in general.

We now derive new theoretical insights into \Matern kernels, ultimately aiming to provide arguments as to why we believe this family of kernels is particularly well suited for surface reconstruction.

\begin{figure}
    \centering
    \subfloat{\includegraphics[width=0.35\textwidth]{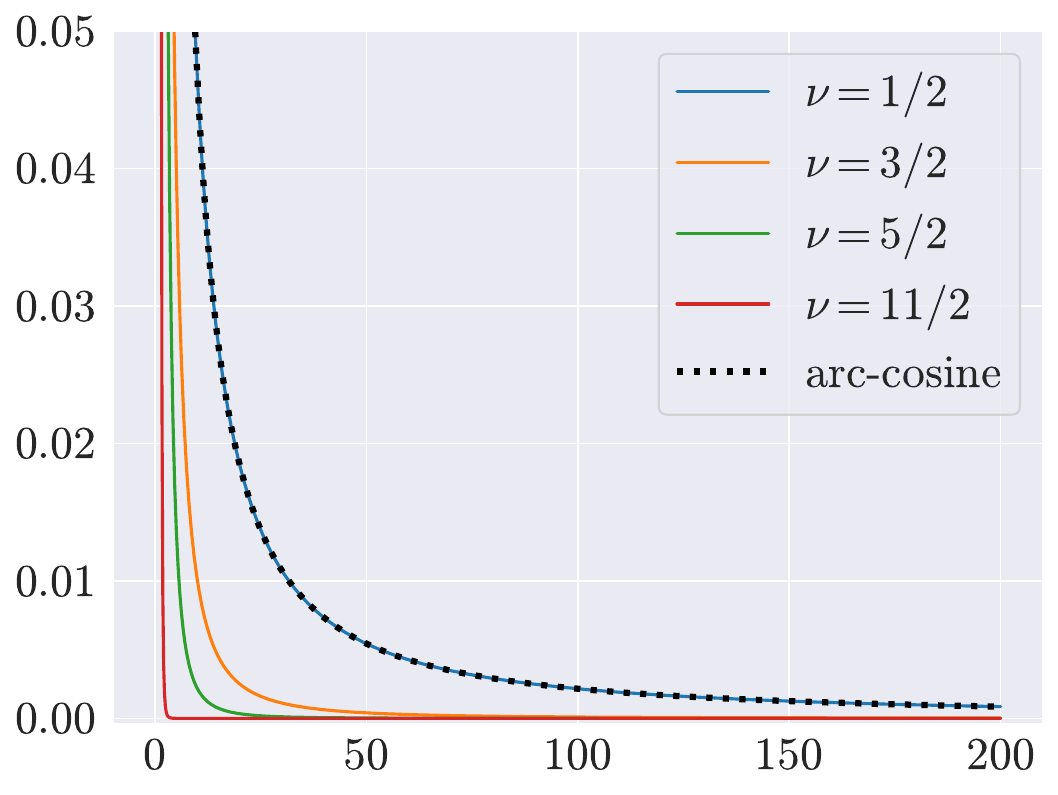}}
    \hspace{.75cm}
    \subfloat{\includegraphics[width=0.35\textwidth]{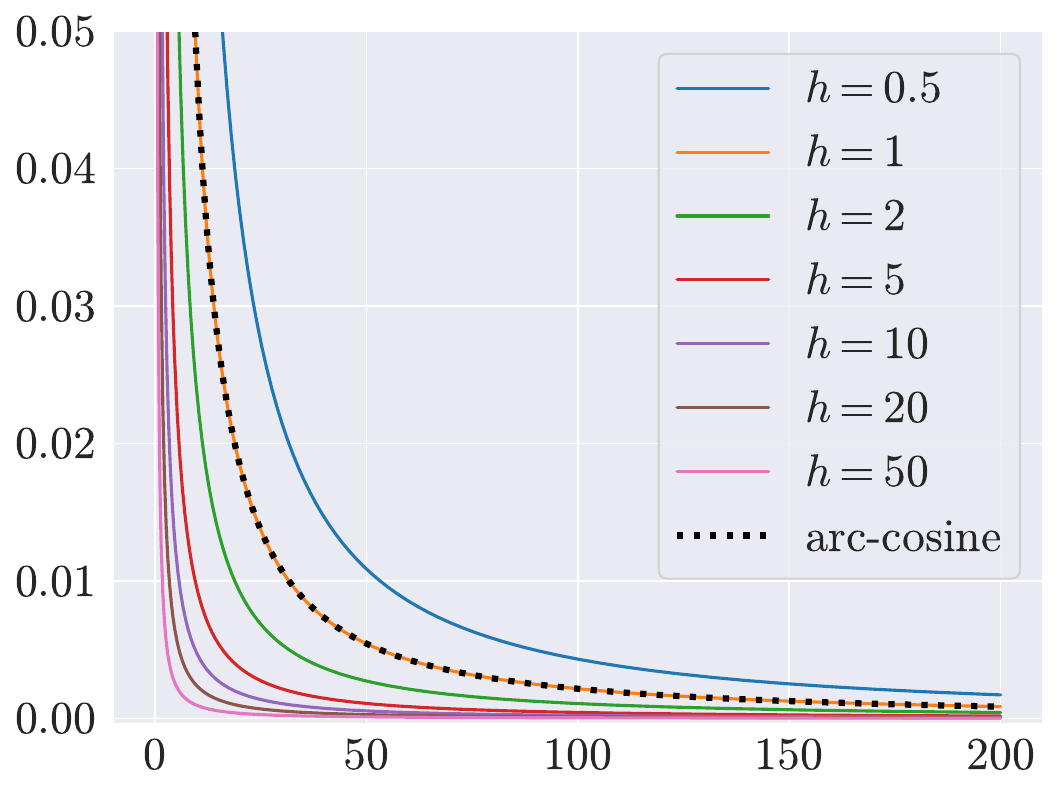}}
    \caption{\textbf{Eigenvalue decay of \Matern kernels.} While we fix $h=1$ and vary $\nu$ on the left, the EDR for $\nu=1/2$ and different values of $h$ is shown on the right. Larger values of $\nu$ and $h$, \textit{i.e.}, smoother kernels, lead to faster eigenvalue decay; hence, slow convergence to high-frequency details.}
    \label{fig:eigenvalue_decay}
\end{figure}

\subsection{Eigenvalue Decay, Spectral Bias, and Reconstruction Error}
\label{subsec:analysis_edr}
\textbf{Eigenvalue decay and spectral bias.}
We begin by showing that \Matern kernels allow for tunable surface reconstruction in the same way as Fourier feature mappings help coordinate-based MLPs overcome spectral bias.
As investigated in \cite{tancik2020}, a rapid decrease of the NTK's eigenvalues implies an associated MLP's slow convergence to high-frequency components of the target function (up to the point where the network is practically unable to learn these components).
Consequently, a smaller eigenvalue decay rate (EDR), \textit{i.e.}, a slower eigenvalue decay, leads to faster convergence to high-frequency content---in the context of implicit surface reconstruction, more detailed geometry.
To overcome this slow convergence, a phenomenon known as \textit{spectral bias}, \cite{tancik2020} use a Fourier feature mapping of the form 
\begin{equation}
    \gamma(x)=[a_1\cos(2\pi b_1^\top x),a_1\sin(2\pi b_1^\top x),\dots,a_q\cos(2\pi b_q^\top x),a_q\sin(2\pi b_q^\top x)]^\top
\end{equation}
applied to the inputs $x\in\mathbb{R}^d$ before passing them to the MLP, effectively transforming the MLP's NTK, $k_\text{NTK}$, into a \textit{stationary} kernel, $k_\text{NTK}(\gamma(x)^\top\gamma(y))=k_\text{NTK}(k_\gamma(x-y))=:k'_\text{NTK}(x-y)$ whose spectrum can be tuned through manipulation of the Fourier basis frequencies, $b_i\in\mathbb{R}^d$, and corresponding coefficients, $a_i\in\mathbb{R}$, of the kernel $k_\gamma(x-y)=\sum_{i=1}^q a^2_i\cos(2\pi b^\top_i(x-y))$.
This shows that $\gamma$ enables explicit control over the NTK's EDR, overcoming spectral bias.

\Matern kernels, being stationary, allow for the same degree of control over the EDR as a Fourier feature mapping; their spectrum can also be tuned (by varying $\nu$ and/or $h$, see Figure \ref{fig:matern_kernels}).
This is in stark contrast to previously employed, non-stationary and parameter-less arc-cosine kernels, whose spectrum is not tunable.
To further study the dependence of the EDR on $\nu$ and $h$, we use:
\begin{theorem}[\cite{seeger2008}, Theorem 3]\label{thm:edr_matern}
   The eigenvalues of \Matern kernels decay polynomially at rate $$\Theta\left(h^{-2\nu}s^{-(1+2\nu/d)}\right)$$ with bandwidth parameter $h>0$, and for \textup{finite} smoothness $0<\nu<\infty$.
\end{theorem}
Figure \ref{fig:eigenvalue_decay} visualizes the EDR for different values of $\nu$ and $h$ along with the EDR of the non-tunable, first-order arc-cosine kernel which equates to $\Theta(s^{-(1+d)/d})$ and has only been very recently proven by \cite{li2024} for general input domains and distributions.
As seen, the smoother the kernel (\textit{i.e.}, the larger $\nu$) and the greater $h$, the faster the eigenvalues decay.
This implies less detailed surface reconstruction (with a higher error) for smoother kernels and larger $h$; indeed, this is exactly what we observe in practice, see Figure \ref{fig:tunable_h}.
Contrary, by lowering both parameters, one can achieve faster convergence to high-frequency geometric details, effectively leading to more nuanced surface reconstruction and a smaller reconstruction error.

\begin{figure}
    \centering
    \includegraphics[width=\textwidth]{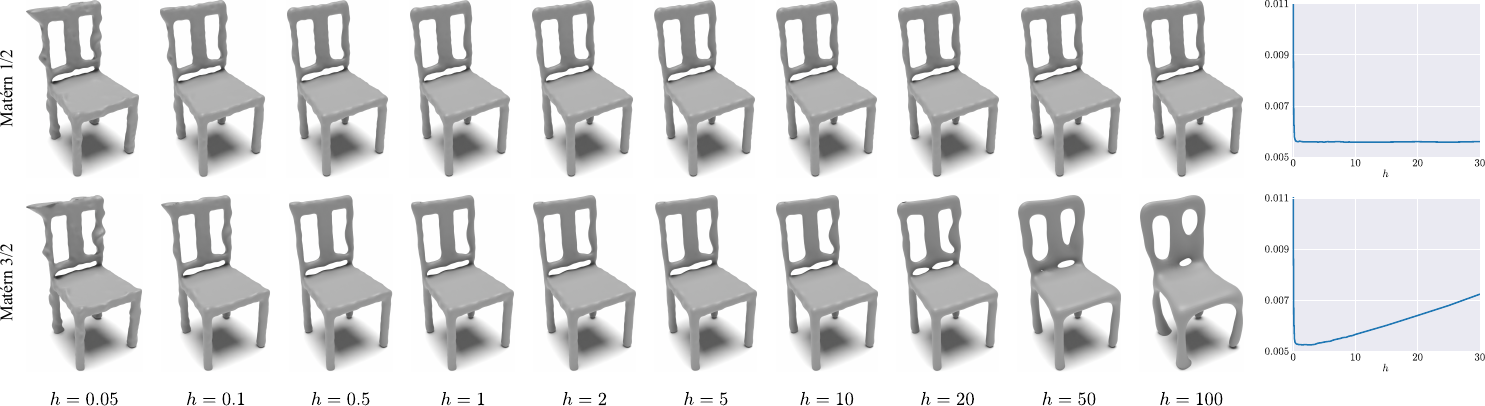}
    \caption{\textbf{\Matern kernels are tunable}. Surface reconstructions can be improved in practice by varying the bandwidth parameter, $h$, effectively tuning the kernels' EDR (see also Figure~\ref{fig:eigenvalue_decay}). However, setting $h$ too small ($<0.5$) results in overfitting, while setting $h$ too big ($>10$) oversmooths (\textit{i.e.}, underfits) the true surface. This is also reflected in the reconstruction error (measured using Chamfer distance), plotted as a function of $h$ on the right. }
    \label{fig:tunable_h}
\end{figure}

\textbf{Reconstruction error.}
We proceed by investigating the bandwidth's influence on the $L_2$ reconstruction error, which is defined (and can be bounded) as
\begin{align}
    \Vert f-\hat{f}\Vert_{L_2}=\left(\int_{\Omega}\left(f(x)-\hat{f}(x)\right)^2\mathrm{d}x\right)^{1/2}\leq C_{\mathcal{X},\Omega}^{(\nu+d)/2}\Vert f\Vert_{H_\nu},
    \label{eq:l2_reconstruction_error}
\end{align}
where $C_{\mathcal{X},\Omega}$ is a constant that only depends on $\mathcal{X}$ and $\Omega$, and 
\begin{equation}
    \Vert f\Vert_{H_\nu}^2=h^{2\nu}\left((2\pi)^{d/2}C_{d,\nu}\right)^{-1}\int_{\mathbb{R}^d}\left(\frac{2\nu}{h^2}+\left(2\pi\Vert\omega\Vert\right)^2\right)^{\nu+d/2}|\mathcal{F}[f](\omega)|^2\mathrm{d}\omega
    \label{eq:rkhs_norm}
\end{equation}
is the norm of the RKHS induced by a \Matern kernel with bandwidth $h$ and smoothness $\nu$.
Moreover, $C_{d,\nu}:=(2^d\pi^{d/2}\Gamma(\nu+d/2)(2\nu)^\nu)/\Gamma(\nu)$, and $\mathcal{F}[f]$ denotes the Fourier transform of a function $f$.
For more information on the bound, including details on $C_{\mathcal{X},\Omega}$, please see, \textit{e.g.}, \cite{santin2016}.
By inspecting Eq. (\ref{eq:l2_reconstruction_error}), we notice that the magnitude of the RKHS norm has a considerably high influence on the reconstruction error.
Based on this observation and the fact that $\Vert f\Vert_{H_\nu}$ depends on $h$, our goal is now to further study the effect of $h$ on the RKHS norm.
As a first result, we have: 
\begin{restatable}{lemma}{lemmaBoundNorm}\label{lemma:bound_norm}
    The RKHS norm of \Matern kernels as defined in Eq. (\ref{eq:rkhs_norm}) can be bounded by $$\Vert f\Vert^2_{H_\nu}\leq h^{2\nu}\left(\frac{1}{h^{2\nu+d}}C_{d,\nu}^1(\mathcal{F}[f])+C^2_{d,\nu}(\mathcal{F}[f])\right),$$ where $h>0$, and $C^1_{d,\nu}$ and $C^2_{d,\nu}$ are functions of $\mathcal{F}[f]$ that do not depend on $h$.
\end{restatable}

A proof can be found in Appendix \ref{ap:sec:proof_bound_norm}.
We immediately see that $\Vert f\Vert^2_{H_\nu}\rightarrow\infty$ in either cases, $h\rightarrow 0$ \textit{and} $h\rightarrow\infty$.
Moreover, based on Proposition \ref{lemma:bound_norm}, it is easy to show that the norm, as a function of $h$, decreases as $h\rightarrow h^*$ from the left, and increases as $h\rightarrow\infty$, where $h^*$ is given as
\begin{equation}
    h^*=\left(\frac{d}{2\nu}Q\right)^{1/(2\nu+d)}\quad\text{with}\quad Q=\frac{C^1_{d,\nu}(\mathcal{F}[f])}{C^2_{d,\nu}(\mathcal{F}[f])}
    \label{eq:optimal_h}
\end{equation}
\begin{wrapfigure}[13]{r}{0.35\textwidth}
    \vspace{-0.4cm}
    \includegraphics[width=0.35\textwidth]{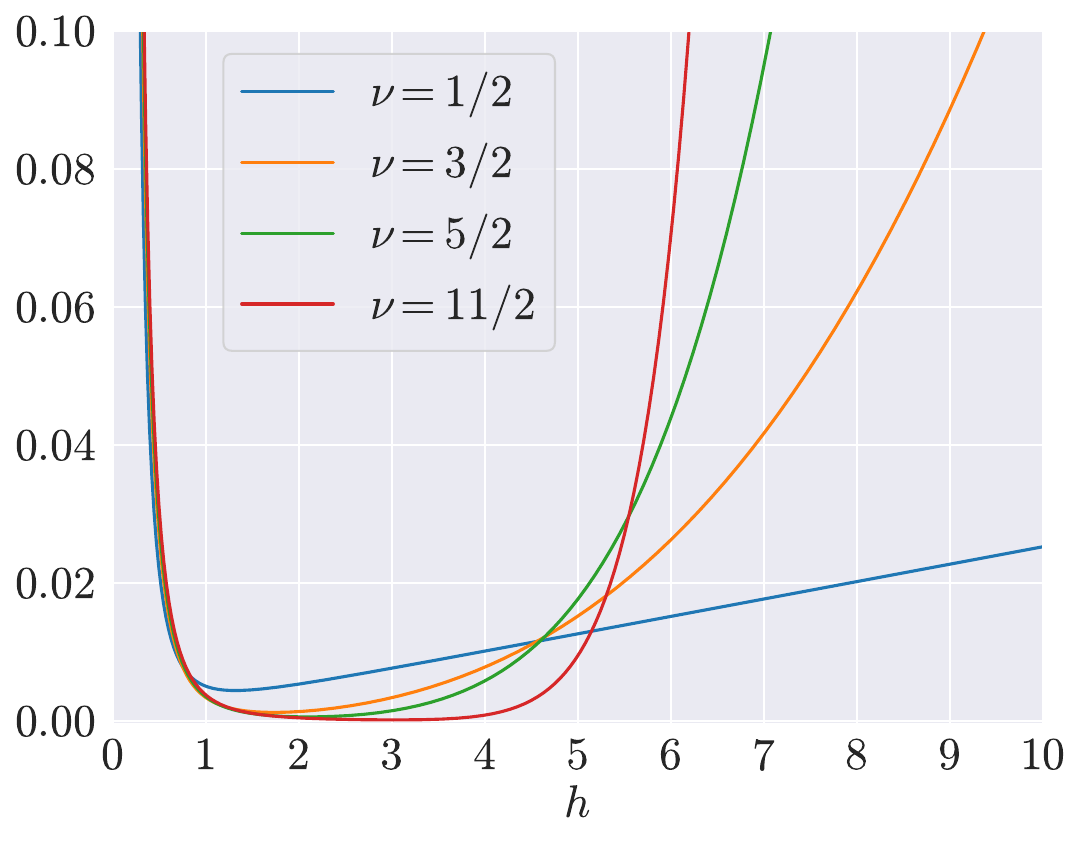}
    \caption{\textbf{RKHS norm as a function of $h$}. We plot the bound from Proposition \ref{lemma:bound_norm} as a function of $h$.}
  \label{fig:rhks_norm_as_h}
\end{wrapfigure}
(see Appendix \ref{ap:sec:derivation_h_opt} for details).
Our analysis shows that the reconstruction error can \textit{not} be made arbitrarily small by just lowering $h$; if $h$ is chosen too small, the reconstruction error starts to increase again, \textit{i.e.}, it \textit{overfits} the true surface.  
Conversely, if $h$ is set too high, the resulting surface is too smooth (\textit{underfits} the true surface), ultimately leading to increased errors.
The optimal trade-off is reached if $h=h^*$.
This is also observed in practice, see Figure \ref{fig:tunable_h}.
As seen from Figure \ref{fig:rhks_norm_as_h}, however, the described effect is more pronounced for smoother \Matern kernels (large $\nu$), as, at some point, the norm (hence, the reconstruction error) increases rapidly for large $h$.
Our theoretical analysis reveals several practical insights.
First, the reconstruction error is generally more sensitive to very small values of $h$.
Second, a good starting point is always $h=1$.

\subsection{Relation to Neural Networks and the Arc-cosine Kernel}
\textbf{Relation to neural networks.}
Next, we study the connection between \Matern kernels and neural networks.
It is well known that the first-order arc-cosine kernel mimics the computation of a two-layer, infinite-width ReLU network $f:\R^d\longrightarrow\R$, $f(x)=m^{-1/2}\sum_{i=1}^m v_i[w^\top_i x+b_i]_+$ when the bottom layer weights $W=(w_1,w_2,\dots,w_m)\in\R^{d\times m}$ and biases $b=(b_1,b_2,\dots,b_m)\in\R^m$ are \textit{fixed} from initialization and drawn from a standard normal distribution, see, \textit{e.g.}, \cite{cho2009,williams2021}.
Here, $[x]_+:=\max\{0, x\}$ denotes the ReLU activation function.

Under similar assumptions, we now show that \Matern kernels mimic the computation of two-layer \textsc{Siren}s \citep{sitzmann2020} if the layer-width approaches infinity. 
We claim the following:

\begin{restatable}{theorem}{thmMaternSirenNtk}\label{thm:matern_siren_ntk}
    Consider a two-layer fully-connected network $f:\R^d\longrightarrow\R$ with sine activation function, $m$ hidden neurons, and fixed bottom layer weights $W=(w_1,w_2,\dots,w_m)\in\R^{d\times m}$ and biases $b=(b_1,b_2,\dots,b_m)\in\R^m$. Let $h,\nu>0$ and $C_{d,\nu}$ be defined as in Eq. (\ref{eq:rkhs_norm}). If $w_i$ is randomly initialized from $$p_\nu(\omega)=h^{-2\nu}C_{d,\nu}\left(\frac{2\nu}{h^2}+(2\pi\Vert\omega\Vert)^2\right)^{-(\nu+d/2)}$$ and $b_i\sim\mathcal{U}(0,2\pi)$, the NTK of $f$ is a \Matern kernel with bandwidth $h$ and smoothness $\nu$ as $m\rightarrow\infty$.
\end{restatable}

A proof can be found in Appendix \ref{ap:sec:proof_matern_siren_ntk}. 
This result establishes for the first time a connection between \textsc{Siren}s and kernel methods.
While explicitly shown for \Matern kernels, as we argue in the Appendix \ref{ap:sec:proof_matern_siren_ntk}, the arguments of Theorem \ref{thm:matern_siren_ntk} in fact apply to all stationary kernels, making it a powerful tool to study the connection between widely used \textsc{Siren}s and any stationary kernel function.

\textbf{\Matern kernels vs. arc-cosine kernels.}
Lastly, we aim to compare \Matern kernels with the previously used first-order arc-cosine kernel.
Our analysis is based on:
\begin{theorem}[\cite{chen2021}, Theorem 1]\label{thm:rkhs_matern_laplace}
    When restricted to the hypersphere $\mathbb{S}^{d-1}$, the RKHS of the \Matern kernel with smoothness $\nu=1/2$ (the Laplace kernel) include the same set of functions as the RKHS induced by the NTK of a fully-connected ($L\geq 2$)-layer ReLU network.
\end{theorem}
In other words, Theorem \ref{thm:rkhs_matern_laplace} shows that the RKHS of the Laplace kernel (a \Matern kernel with smoothness $\nu=1/2$) is the same as the RKHS induced by the NTK of a fully-connected ReLU network when inputs are restricted to $\mathbb{S}^{d-1}$.
Based on Theorem~\ref{thm:rkhs_matern_laplace}, we obtain the following connection between Laplace and the first-order arc-cosine kernel:
\begin{restatable}{corollary}{corRkhsMaternArcCosine}\label{cor:rkhs_matern_arc-cosine}
    The RKHS of the Laplace kernel is equal to the RKHS induced by a sum of zeroth and first-order arc-cosine kernels, implying that, when restricted to $\mathbb{S}^{d-1}$, the Laplace and first-order arc-cosine kernel are equal up to constant scaling and addition of another kernel function.
\end{restatable}
See Appendix \ref{ap:sec:proof_rkhs_matern_arc-cosine} for a proof. 
For general input domains, our empirical results presented next support (at least for ($L=2$)-layer networks) the widely believed claim that the NTK is not significantly different from standard kernels \citep{belkin2018,hui2019,geifman2020}, such as the Laplace kernel.
Indeed, we show that the Laplace kernel even outperforms the arc-cosine kernel.

\section{Experiments and Results}
We systematically evaluated the effectiveness of \Matern kernels in the context of implicit surface reconstruction, presenting results on ShapeNet~\citep{chang2015} and the \textit{Surface Reconstruction Benchmark} (SRB;~\citet{berger2013}) in Section \ref{subsec:results_shapenet}. 
Moreover, Section \ref{subsec:texture_recon} demonstrates \Matern kernels' ability to reconstruct high-frequency textures on the \textit{Google Scanned Objects} (GSO; \cite{downs2022}) dataset and Objaverse \citep{deitke2023}.
Leveraging Neural Kernel Fields~\citep{williams2022}, we present and evaluate learnable \Matern kernels in Section \ref{subsec:learnable_kernels}.

\subsection{Surface Reconstruction on ShapeNet and SRB}
\label{subsec:results_shapenet}
\textbf{ShapeNet}. We compare \Matern kernels in a sparse setting against the arc-cosine kernel on ShapeNet (using train/val/test split provided by~\citet{mescheder2019}).
To do so, we randomly sample $m=1,000$ surface points with corresponding normals for each shape.
We implemented \Matern kernels in PyTorch and took the official CUDA implementation of the arc-cosine kernel from NS, eventually integrated into a \textit{unified} framework to ensure fair comparison.
Notably, we did not use Nyström sampling as in NS and employed PyTorch's direct Cholesky solver to solve the KRR problem in Eq. (\ref{eq:emp_risk}) instead of an iterative conjugate gradient solver. 
Moreover, similar to NS, we approximate the gradient part in Eq. (\ref{eq:emp_risk}) with finite differences.
See Appendix \ref{ap:sec:shapenet_non_learnable_details} for details.

\begin{table}
\begin{adjustbox}{width=\linewidth}
\renewcommand{\arraystretch}{1.017}
\begin{tabular}{lccc|ccc}
\toprule
& \multicolumn{3}{c|}{F-Score $\uparrow$} & \multicolumn{3}{c}{CD $\downarrow$} \\
\cmidrule(lr){2-4} \cmidrule(lr){4-7}
& 0.5 & 1 & 2 & 0.5 & 1 & 2 \\
\midrule
\Matern 1/2 & 93.6 & 93.7 & 93.7 & 4.05 & \underline{4.02} & \underline{4.02} \\
\Matern 3/2 & 94.6 & \underline{94.8} & \textbf{94.9} & 4.05 & \textbf{4.00} & 4.09 \\
\Matern 5/2 & 92.4 & 93.8 & 92.9 & 6.42 & 5.65 & 6.91 \\
\Matern $\infty$ & 59.3 & 52.4 & 40.9 & 57.56 & 50.59 & 48.38 \\
\midrule
Arc-cosine & \multicolumn{3}{c|}{92.8} & \multicolumn{3}{c}{4.67} \\
\midrule
NS & \multicolumn{3}{c|}{90.6} & \multicolumn{3}{c}{4.74} \\
SPSR & \multicolumn{3}{c|}{84.3} & \multicolumn{3}{c}{6.26} \\
\bottomrule
\end{tabular}
\renewcommand{\arraystretch}{1} 
\hspace{0.4cm}
\resizebox{!}{\heightof{
\begin{tabular}{lccc|ccc}
\toprule
& \multicolumn{3}{c|}{F-Score $\uparrow$} & \multicolumn{3}{c}{CD $\downarrow$} \\
\cmidrule(lr){2-4} \cmidrule(lr){4-7}
& 0.5 & 1 & 2 & 0.5 & 1 & 2 \\
\midrule
\Matern 1/2 & 93.6 & 93.7 & 93.7 & 4.05 & \underline{4.02} & \underline{4.02} \\
\Matern 3/2 & 94.6 & \underline{94.8} & \textbf{94.9} & 4.05 & \textbf{4.00} & 4.09 \\
\Matern 5/2 & 92.4 & 93.8 & 92.9 & 6.42 & 5.65 & 6.91 \\
\Matern $\infty$ & 59.3 & 52.4 & 40.9 & 57.56 & 50.59 & 48.38 \\
\midrule
Arc-cosine & \multicolumn{3}{c|}{92.8} & \multicolumn{3}{c}{4.67} \\
\midrule
NS & \multicolumn{3}{c|}{90.6} & \multicolumn{3}{c}{4.74} \\
SPSR & \multicolumn{3}{c|}{84.3} & \multicolumn{3}{c}{6.26} \\
\bottomrule
\end{tabular}
}}{%
\begin{tabular}{lcc}
\toprule
& CD $\downarrow$ & HD $\downarrow$ \\
\midrule
\Matern 1/2 & 0.21 & 4.43 \\
\Matern 3/2 & \underline{0.18} & \underline{2.93} \\
\Matern 5/2 & 0.58 & 22.52 \\
\Matern $\infty$ & 3.83 & 33.27 \\
\midrule
NS* & 0.19 & 3.19 \\
\midrule
NS & \textbf{0.17} & \textbf{2.85} \\
SPSR & 0.21 & 4.69 \\
FFN & 0.28 & 4.45 \\
\textsc{Siren} & 0.19 & 3.86 \\
SAP & 0.21 & 4.85 \\
DiGS & \underline{0.18} & 3.55 \\
VisCo & \underline{0.18} & 2.95 \\
OG-INR & 0.20 & 4.06 \\
\bottomrule
\end{tabular}
}
\end{adjustbox}
\caption{\textbf{Results on ShapeNet (left) and SRB (right) for non-learnable kernels}. Arc-cosine kernel and \Matern $\nu\in\{\text{1/2},\text{3/2},\text{5/2},\infty\}$ are implemented in a unified framework and based on the exact same parameters. "NS*" denotes the best result we could achieve by running the official implementation of NS, see Appendix \ref{ap:sec:details_srb}. \textbf{Bold} marks best result, \underline{underline} second best.}
\label{tab:results_non_learnable}
\end{table}

Quantitative results in terms of F-Score (with a threshold of 0.01) and Chamfer distance (CD; always reported $\times$ 10$^3$) can be found in Table \ref{tab:results_non_learnable}.
\Matern 1/2 and 3/2 consistently outperform the arc-cosine kernel as well as popular SPSR~\citep{kazhdan2013}.
In addition, our experiments show that the bandwidth parameter, $h$, can be used to tune the surface reconstruction (lower the error), while not being too sensitive in practice. 
\Matern 5/2 and the Gaussian kernel (for $\nu\rightarrow\infty$) perform significantly worse, being simply too smooth.
Figure~\ref{fig:recons_ns_matern} and Appendix \ref{ap:sec:shapenet_non_learnable_details} shows qualitative results.

\begin{wraptable}{r}{3.7cm}
\vspace{-0.41cm}
\centering
\begin{tabular}{ll}
\toprule
& Time\\
\midrule
\Matern 1/2 & \underline{9.85} \\
\Matern 3/2 & 12.47\\
\Matern 5/2 & 15.26 \\
\Matern $\infty$ & 12.76 \\
\midrule
NS* & 19.89 \\
\midrule
NS & 11.91 \\
SPSR & \textbf{1.65} \\
\bottomrule
\end{tabular}
\caption{\textbf{Runtime comparison on SRB}. Time in seconds.}
\label{tab:runtime_srb}
 \end{wraptable}

\textbf{SRB}. Next, we evaluate \Matern kernels on the challenging Surface Reconstruction Benchmark which consists of five complex shapes simulated from incomplete and noisy range scans with up to 100,000 points.
For this, we implemented a highly scalable version of \Matern kernels, closely following the NS framework to solve the KRR problem in Eq. (\ref{eq:emp_risk}) using Nyström sampling (with 15,000 samples) and the FALKON~\citep{rudi2017} solver.
We compare against NS and SPSR, as well as kernel-free methods, Fourier Feature Networks (FFNs; \cite{tancik2020}), \textsc{Siren}, SAP \citep{peng2021}, DiGS \citep{shabat2022}, VisCo \citep{pumarola2022}, and OG-INR \citep{koneputugodage2023}.
All methods are optimization-based (FFN and \textsc{Siren} are overfitted to a single shape) and use surface normals.
For \Matern kernels, we did a modest parameter sweep over $h\in\{0.5,1,2\}$ and took the reconstruction with the lowest Chamfer distance.
Runtime is measured on a single NVIDIA V100.

\begin{figure}
    \centering
    \includegraphics[width=1\linewidth]{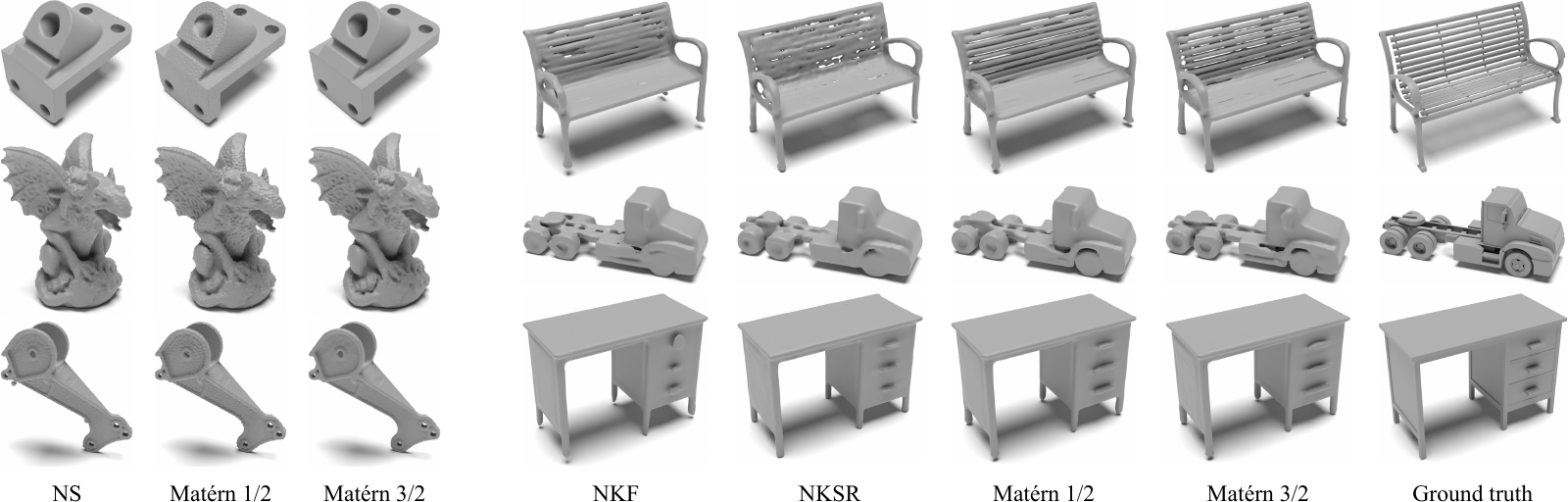}
    \caption{\textbf{Qualitative results for non-learnable kernels on SRB (left) and learned kernels on ShapeNet (right)}. Please see Appendices \ref{ap:sec:details_srb} and \ref{ap:sec:shapenet_learnable_additional_results} for more qualitative results.}
    \label{fig:qualitative_results}
\end{figure}

Quantitative results obtained by employing Chamfer and Hausdorff distance (HD) can be found in Table~\ref{tab:results_non_learnable}, demonstrating that \Matern 3/2 performs on par with NS while being significantly faster to compute (compared to what \textit{we} measured for NS), see Table \ref{tab:runtime_srb}.
Moreover, \Matern 3/2 outperforms all evaluated state-of-the-art kernel-free methods (FFN, \textsc{Siren}, SAP, DiGS, VisCo, and OG-INR).
Figure \ref{fig:qualitative_results} (left) and Appendix \ref{ap:sec:details_srb} provide more qualitative results and per-object metrics.

\subsection{Texture Reconstruction on GSO and Objaverse}
\label{subsec:texture_recon}
\begin{wraptable}{r}{5.8cm}
\vspace{-0.35cm}
\begin{adjustbox}{width=\linewidth}
\begin{tabular}{llcc}
\toprule
$m$ & Kernel & PSNR $\uparrow$ & LPIPS $\downarrow$ \\
\midrule
\multirow{2}{*}{10,000} & \Matern 3/2 & \textbf{19.61} & \textbf{2.05} \\ & Arc-cosine & 19.34 & 2.07 \\
\midrule
\multirow{2}{*}{2,500} & \Matern 3/2 & \textbf{18.92} & 2.15 \\ & Arc-cosine & 18.88 & \textbf{2.09} \\
\bottomrule
\end{tabular} 
\end{adjustbox}
\caption{\textbf{Texture reconstruction on GSO}. \Matern kernels outperform the arc-cosine kernel in dense and sparse settings.}
\label{tab:quant_texture_eval}
\end{wraptable}
Lastly, we demonstrate \Matern kernels' ability to represent other \textit{high-frequency} scene attributes, such as texture.
To do so, we randomly sample $m$ surface points with corresponding normals and per-point RGB color values from textured meshes, yielding an extended dataset $\mathcal{D}'=\mathcal{D}\times\{c_1,c_2,\dots,c_m\}\subset\Omega\times\mathbb{R}^d\times\mathbb{R}^3$ for each shape.
Then, instead of modeling an SDF $f$ as in Eq. (\ref{eq:emp_risk}), we are seeking a function $f':\mathbb{R}^d\longrightarrow\mathbb{R}^4$ that, along with signed distances, also predicts per-point RGB values.
Please find more information in Appendix \ref{ap:sec:texture_recon} about how we adapt the KRR problem in Eq. (\ref{eq:emp_risk}) to estimate $\hat{f}'$.
Finally, we extract the object's surface using Marching Cubes and trilinearily interpolate RGB values at previously predicted surface points. 
\begin{wrapfigure}[15]{r}{0.58\textwidth}
    \vspace{-0.48cm}
    \includegraphics[width=0.58\textwidth]{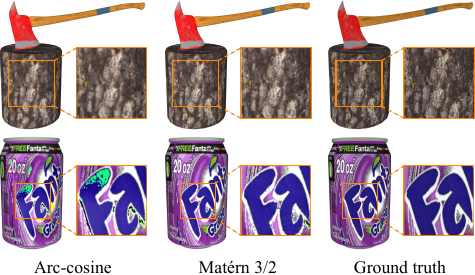}
    \caption{\textbf{Texture reconstruction on Objaverse}. \Matern kernels lead to fewer artifacts and sharper details than the arc-cosine kernel.}
    \label{fig:texture_recon}
\end{wrapfigure}
We present quantitative results in terms of PSNR and LPIPS on GSO in Table \ref{tab:quant_texture_eval} (metrics evaluated on the texture atlas; see Appendix \ref{ap:sec:texture_recon} for further details).
\Matern 3/2 surpasses the arc-cosine kernel in the densely sampled setting using 10,000 surface points, as well as in the sparse setting.
Notably, we did not tune \Matern kernels' bandwidth parameter for this experiment. 

We also show qualitative results on the challenging Objaverse \citep{deitke2023} dataset in Figure \ref{fig:texture_recon}.
We chose Objaverse as it includes extremely high-resolution and complex textures.
\Matern 3/2 reconstructs high-frequency texture details with great precision, overall yielding visually more pleasant reconstructions than the arc-cosine kernel.
Reconstructed textures have fewer artifacts and are generally sharper, demonstrating that \Matern kernels can overcome spectral bias much better than the arc-cosine kernel. 
Note that this perfectly confirms our theoretical analysis presented in Section \ref{subsec:analysis_edr}. 

\subsection{Data-dependent \Matern Kernels}
\label{subsec:learnable_kernels}
We leverage the \textit{Neural Kernel Field} (NKF) framework introduced by \citet{williams2022} to make \Matern kernels learnable.
Specifically, NKF proposed to feed points $x,y\in\mathbb{R}^d$ through an input-conditioned neural network $\varphi:\mathbb{R}^d\longrightarrow\mathbb{R}^q$ before evaluating a kernel function, $k^\varphi(x,y;\mathcal{D})=k\left([x,\varphi(x;\mathcal{D})],[y,\varphi(y;\mathcal{D})]\right)$.
We re-implemented the NKF framework since the authors did not provide code.
We set $q=32$ for our experiments.
Please see \cite{williams2022} for details.

\textbf{Sparse surface reconstruction and extreme generalization.}
We compare learned \Matern kernels against the original NKF framework (with the arc-cosine kernel) and NKSR~\citep{huang2023} in a sparse setting on ShapeNet.
Again, we sample $m=1,000$ surface points and corresponding normals for each shape and set $h=1$ for all \Matern kernels.
Additionally, we evaluate \Matern kernels' out-of-category generalization ability in an extreme setting, in which we train only on chairs and evaluate on the other 12 ShapeNet categories.
Table \ref{tab:shapenet_results} reports the results, quantified using intersection-over-union (IoU), Chamfer distance, and normal consistency (NC).
Learned \Matern 1/2 and 3/2 outperform the arc-cosine kernel while being significantly faster to train. 
Training NKF on the entire ShapeNet dataset (consisting of approx. 30,000 shapes) takes about six days---almost twice as long as for \Matern kernels, which require about three days (measured on a single NVIDIA A100).
Additionally, although \Matern 1/2 is not able to surpass NKSR, it comes very close (95.3 vs. 95.6 in IoU) while having a shorter training time.
Regarding out-of-category generalization, we observe that \Matern 1/2 performs best, followed by \Matern 3/2, NKF, and NKSR.
Please see Figure \ref{fig:qualitative_results} (right) and Appendix \ref{ap:sec:shapenet_learnable_additional_results} for qualitative results.

\begin{table}
\begin{adjustbox}{width=\linewidth}
{\renewcommand{\arraystretch}{1.05}%
\begin{tabular}[t]{lccc}
\toprule
& IoU $\uparrow$ & CD $\downarrow$ & NC $\uparrow$ \\
\midrule
\Matern 1/2 & \underline{95.3} & \underline{2.58} & \textbf{95.6} \\
\Matern 3/2 & 94.9 & 2.70 & 95.3 \\
\Matern 5/2 & 93.3 & 3.07 & 94.9 \\
\Matern $\infty$ & 92.1 & 3.39 & 94.2 \\
\midrule
NKF* & 94.7 & 2.70 & 95.2 \\
\midrule
NKF & 94.7 & 2.65 & \underline{95.4} \\
NKSR & \textbf{95.6} & \textbf{2.34} & \underline{95.4} \\
\bottomrule
\end{tabular}}
\hspace{0.5cm}
\begin{tabular}[t]{lccc||c}
\toprule
\multicolumn{5}{c}{\textbf{Train only on chairs, test on all}} \\
\cmidrule(lr){1-5}
& IoU $\uparrow$ & CD $\downarrow$ & NC $\uparrow$ & Time/epoch \\
\midrule
\Matern 1/2 &\textbf{93.4} & \underline{3.06} & \textbf{94.3} & \textbf{7.71 min} \\
\Matern 3/2 & \underline{92.8} & 3.34 & \underline{94.2} & 8.26 min \\
\Matern 5/2 & 90.5 & 4.11 & 93.9 & 8.56 min \\
\Matern $\infty$ & 84.7 & 6.70 & 92.6 & \underline{7.74 min} \\
\midrule
NKF* & \underline{92.8} & 3.30 & 94.1 & 31.98 min \\
\midrule
NKSR & 89.6 & \textbf{2.70} & 94.1 & 41.66 min \\
\bottomrule
\end{tabular}
\end{adjustbox}
\caption{\textbf{Results on ShapeNet for learned kernels}. "NKF*" denotes a re-implemented variant of NKF. NKF* and \Matern $\nu\in\{\text{1/2},\text{3/2},\text{5/2},\infty\}$ are based on the same framework and parameters; they differ \textit{only} in the employed kernel. Runtime is measured on a single NVIDIA A100 with a batch size of one to ensure fair comparison. \textbf{Bold} marks best result, \underline{underline} second best.}
\label{tab:shapenet_results}
\end{table}

\begin{table}
\begin{adjustbox}{width=\linewidth}
\begin{tabular}{lccc|ccc||ccc|ccc||ccc|ccc}
\toprule
& \multicolumn{6}{c||}{No noise ($\sigma=0$)} & \multicolumn{6}{c||}{Small noise ($\sigma=0.0025$)} & \multicolumn{6}{|c}{Big noise ($\sigma=0.005$)} \\
\cmidrule(lr){2-19}
& \multicolumn{3}{c|}{IoU $\uparrow$} & \multicolumn{3}{c||}{CD $\downarrow$} & \multicolumn{3}{c|}{IoU $\uparrow$} & \multicolumn{3}{c||}{CD $\downarrow$} & \multicolumn{3}{c|}{IoU $\uparrow$} & \multicolumn{3}{c}{CD $\downarrow$} \\
\cmidrule(lr){2-19} 
& 0.5 & 1 & 2 & 0.5 & 1 & 2 & 0.5 & 1 & 2 & 0.5 & 1 & 2 & 0.5 & 1 & 2 & 0.5 & 1 & 2 \\
\midrule
\Matern 1/2 & 93.3 & \textbf{93.6} & \underline{93.5} & 2.88 & \underline{2.85} & 2.87 & 91.9 & \textbf{92.1} & \textbf{92.1} & \underline{3.03} & 3.04 & 3.07 & \underline{89.5} & \textbf{89.6} & \textbf{89.6} & \underline{3.37} & \textbf{3.36} & 3.39 \\
\Matern 3/2 & 92.7 & 92.7 & 92.3 & 3.08 & 3.20 & 3.25 & 89.7 & 89.8 & 89.7 & 3.48 & 3.69 & 3.54 & 85.8 & 85.5 & 86.3 & 4.23 & 4.32 & 4.15 \\
\Matern 5/2 & 91.8 & 90.5 & 88.8 & 3.58 & 3.69 & 4.18 & 87.9 & 87.6 & 87.0 & 4.29 & 4.48 & 4.43 & 83.4 & 82.2 & 84.3 & 4.90 & 5.03 & 5.11 \\
\Matern $\infty$ & 86.4 & 85.8 & 86.8 & 5.80 & 5.07 & 4.75 & 81.3 & 83.3 & 81.0 & 7.27 & 5.62 & 6.45 & 76.7 & 80.7 & 79.3 & 8.27 & 5.90 & 6.67 \\
\midrule
NKF* & \multicolumn{3}{c|}{93.2} & \multicolumn{3}{c||}{2.97} & \multicolumn{3}{c|}{\underline{92.0}} & \multicolumn{3}{c||}{3.12} & \multicolumn{3}{c|}{\underline{89.5}} & \multicolumn{3}{c}{3.39} \\
\midrule
NKSR & \multicolumn{3}{c|}{91.1} & \multicolumn{3}{c||}{\textbf{2.65}} & \multicolumn{3}{c|}{90.1} & \multicolumn{3}{c||}{\textbf{2.98}} & \multicolumn{3}{c|}{88.4} & \multicolumn{3}{c}{3.41} \\
\bottomrule
\end{tabular}
\end{adjustbox}
\caption{\textbf{Robustness against noise}. We compare \Matern kernels' robustness against noise on a subset of the ShapeNet dataset using different noise levels, $\sigma\in\{0,0.0025,0.005\}$ and bandwidths, $h\in\{0.5,1,2\}$. \textbf{Bold} marks best result, \underline{underline} second best.}
\label{tab:robsustness_noise}
\end{table}

\textbf{Robustness against noise.}
We evaluate \Matern kernel's robustness against different noise levels, $\sigma\in\{0,0.0025,0.005\}$, on a subset of the ShapeNet dataset which includes approximately 1,700 shapes. 
To construct the dataset, we downsampled each ShapeNet category to include only 5\% of the shapes.
Table \ref{tab:robsustness_noise} presents the results.
NKSR, being specifically optimized to deal with noisy inputs, degrades the least with increasing noise level (3\% in IoU from no to big noise versus 4.3\% for \Matern 1/2).
Moreover, \Matern 1/2 is slightly more robust against noise than NKF. 
Generally, varying the bandwidth, $h$, helps increase the robustness against noise. 

\begin{wrapfigure}[9]{r}{0.3\textwidth}
\vspace{-0.6cm}
  \begin{center}
   \includegraphics[width=0.3\textwidth]{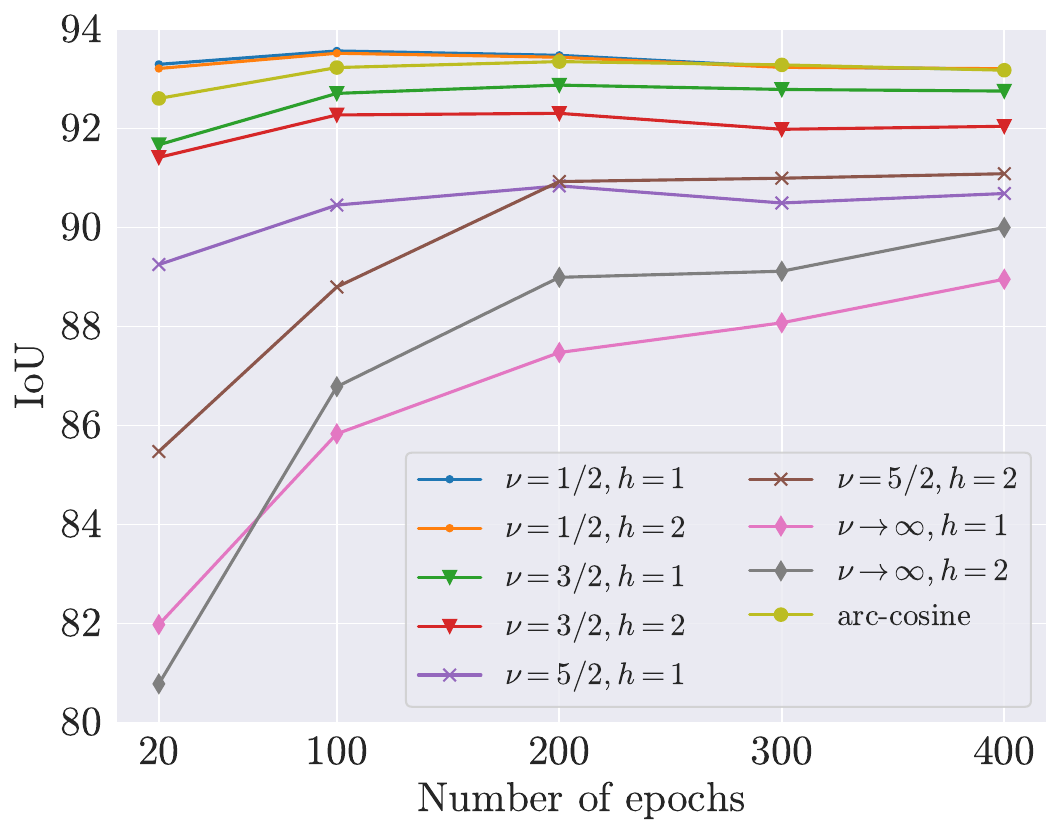}
  \end{center}
\end{wrapfigure}
\textbf{Convergence speed.}
Finally, we analyze \Matern kernels' convergence speed in comparison to the arc-cosine kernel as implemented in NKF, see inset figure on the right.
As exemplified, \Matern 1/2 already converges after just 100 epochs, while the arc-cosine kernel requires twice as many epochs, also never being able to reach the same top accuracy.
In general, we see that the smoother the kernel, the slower the convergence---an observation that perfectly matches our theoretical analysis in Section~\ref{subsec:analysis_edr}.
\Matern 1/2 converges the fastest, followed by \Matern 3/2, 5/2, and the Gaussian kernel.

\section{Conclusion}
In this work, we have proposed to use the family of \Matern kernels for implicit surface reconstruction and showed that it consistently outperforms the recently employed first-order arc-cosine kernel---both, in a non-learnable as well as learnable regime---while being significantly easier to implement, faster to compute and train, and highly scalable.
We demonstrated that \Matern kernels lead to \textit{tunable} surface reconstruction, and, based on a new bound of the $L_2$ reconstruction error, derive practical insights into how to choose their tunable bandwidth parameter.
Moreover, we presented an in-depth theoretical analysis, analyzing \Matern kernels' relation to widely used Fourier feature mappings, \textsc{Siren} networks, and arc-cosine kernels.

\subsubsection*{Acknowledgments}
We thank Francis Williams for valuable discussions and his support in re-implementing the NKF framework and Lukas Meyer for his help during NKSR evaluation.
This work was funded by the German Federal Ministry of Education and Research (BMBF), FKZ: 01IS22082 (IRRW). 
The authors are responsible for the content of this publication.
The authors gratefully acknowledge the scientific support and HPC resources provided by the Erlangen National High Performance Computing Center (NHR@FAU) of the Friedrich-Alexander-Universität Erlangen-Nürnberg (FAU) under the NHR project b112dc IRRW. NHR funding is provided by federal and Bavarian state authorities. NHR@FAU hardware is partially funded by the German Research Foundation (DFG) – 440719683.

\bibliography{iclr2025_conference}
\bibliographystyle{iclr2025_conference}

\clearpage
\appendix

\section{On the Stationarity of \Matern Kernels}
\label{ap:sec:translation_matern}
As noted in Section \ref{subsec:basic_props} of the main paper, \Matern kernels are \textit{stationary}: they only depend on the difference $x-y$ between two points $x,y\in\Omega$, hence being translation invariant.
This is because $k(x-y)=k(x+t-(y+t))$ for $t\in\mathbb{R}^d$ and all stationary kernels, $k$.
In contrast, the arc-cosine kernel is not translation invariant as its value depends on the absolute positions, $x,y$, see Eq. (\ref{eq:def_arc_cosine}) of the main paper (\textit{i.e.}, the arc-cosine kernel can not be written as a function of $x-y$).
In fact, the arc-cosine kernel is \textit{non-stationary}.
In the following and with the help of Figure \ref{fig:translation_invariance}, we will further discuss why stationary kernels, such as \Matern kernels, might be a better choice for kernel-based implicit surface reconstruction than non-stationary kernels, such as the arc-cosine kernel.

First, stationary kernels are \textit{naturally} translation invariant. 
Consequently, they are independent of the absolute embedding of the input points---translating them does not change the reconstructed surface as the kernel's value only depends on the \textit{relative} distances between points. 
This is highly beneficial from a practical perspective since no additional steps (such as centering the point cloud) have to be taken before surface reconstruction. 
Being non-stationary, the arc-cosine kernel behaves differently.
If we would just use the kernel as in Eq. (\ref{eq:def_arc_cosine}) of the main paper, the reconstructed surface would vary depending on the actual location of the input points in space. 
This is a highly undesirable property when it comes to surface reconstruction. 
To make the arc-cosine kernel translation invariant, input points must be centered before reconstruction (thus eliminating the effects of translation).
A centered arc-cosine kernel, denoted as $\Bar{k}_\text{AC}$, reads:
\begin{equation*}
    \Bar{k}_\text{AC}(x,y)=\frac{\Vert x-\Bar{c}\Vert\Vert y-\Bar{c}\Vert}{\pi}\left(\sin\Bar{\theta}+(\pi-\Bar{\theta})\cos\Bar{\theta}\right),\quad\text{where}\quad\Bar{\theta}=\cos^{-1}\left(\frac{(x-\Bar{c})^\top(y-\Bar{c})}{\Vert x-\Bar{c}\Vert\Vert y-\Bar{c}\Vert}\right).
\end{equation*}
Here, $c\in\mathbb{R}^d$ denotes the centroid of the input point cloud. 
$\Bar{k}_\text{AC}$ is obviously translation invariant because $\Vert x+t-(\Bar{c}+t)\Vert \Vert y+t-(\Bar{c}+t)\Vert=\Vert x-\Bar{c}\Vert\Vert y-\Bar{c}\Vert$ and $(x+t-(\Bar{c}+t))^\top(y+t-(\Bar{c}+t))=(x-\Bar{c})^\top(y-\Bar{c})$ for all $t\in\mathbb{R}^d$.
This follows because translating the input points also shifts its centroid.
In conclusion, while stationary \Matern kernels can be directly used for kernel-based surface reconstruction, additional pre-processing steps are necessary for the arc-cosine kernel.
\begin{figure}[b!]
    \includegraphics[width=\linewidth]{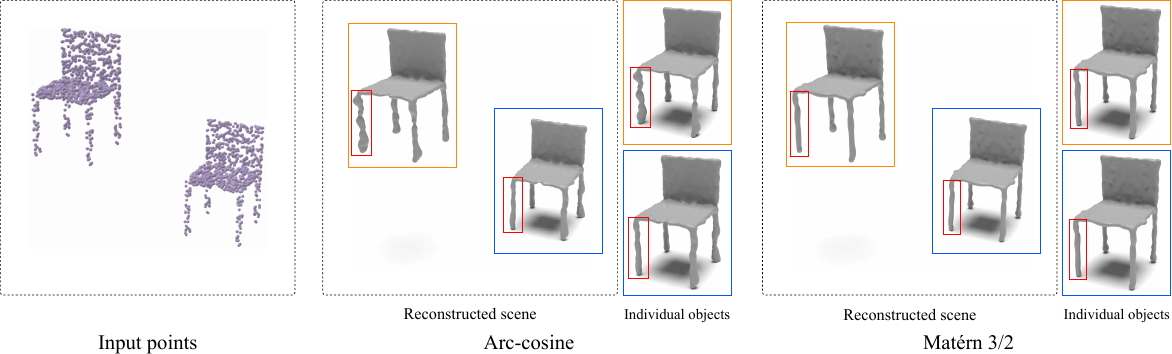}
    \caption{\textbf{\Matern kernels are translation invariant and lead to locally consistent surface reconstructions}. Here, we show that two identical chairs in a scene are reconstructed differently depending on the absolute position of the input points when using the non-stationary arc-cosine kernel (especially noticeable at the chair's legs). In contrast, stationary \Matern kernels lead to locally consistent surface reconstructions, being independent of the actual embedding (\textit{i.e.}, both chairs have the exact same reconstructed surface).}
    \label{fig:translation_invariance}
\end{figure}

Second, stationary kernels lead to \textit{locally consistent} surface reconstructions.
Since a stationary kernel’s value only depends on the \textit{relative} positions, objects in a scene (or parts of objects) with similar geometric properties will be reconstructed consistently as long as the relative distances between points on each instance (or part) remain the same. 
This is a crucial and very desirable property for surface reconstruction which is \textit{not} shared by the arc-cosine kernel (because it is non-stationary).

\section{Proof of Proposition \ref{lemma:bound_norm}}
\label{ap:sec:proof_bound_norm}

\lemmaBoundNorm*

\begin{proof}
    Based on a simple bound
    \begin{align*}
        I&=\int_{\mathbb{R}^d}\left(\frac{2\nu}{h^2}+\left(2\pi\Vert\omega\Vert\right)^2\right)^{\nu+d/2}|\mathcal{F}[f](\omega)|^2\mathrm{d}\omega \\
        &\leq\int_{\mathbb{R}^d}\left(\frac{2\nu}{h^2}\right)^{\nu+d/2}|\mathcal{F}[f](\omega)|^2\mathrm{d}\omega+\int_{\mathbb{R}^d}\left(\left(2\pi\Vert\omega\Vert\right)^2\right)^{\nu+d/2}|\mathcal{F}[f](\omega)|^2\mathrm{d}\omega \\
        &=\left(\frac{2\nu}{h^2}\right)^{\nu+d/2}\underbrace{\int_{\mathbb{R}^d}|\mathcal{F}[f](\omega)|^2\mathrm{d}\omega}_{=:C'(\mathcal{F}[f])}+\underbrace{\int_{\mathbb{R}^d}\left(\left(2\pi\Vert\omega\Vert\right)^2\right)^{\nu+d/2}|\mathcal{F}[f](\omega)|^2\mathrm{d}\omega}_{=:C'_{d,\nu}(\mathcal{F}[f])} \\
        &=\left(\frac{2\nu}{h^2}\right)^{\nu+d/2}C'(\mathcal{F}[f])+C'_{d,\nu}(\mathcal{F}[f])
    \end{align*}
    we get
    \begin{align*}
        \Vert f\Vert_{H_\nu}^2&=h^{2\nu}\left((2\pi)^{d/2}C_{d,\nu}\right)^{-1}I \\
        &\leq h^{2\nu}\left((2\pi)^{d/2}C_{d,\nu}\right)^{-1}\left(\left(\frac{2\nu}{h^2}\right)^{\nu+d/2}C'(\mathcal{F}[f])+C'_{d,\nu}(\mathcal{F}[f])\right) \\
        &=h^{2\nu}\left(\frac{1}{h^{2\nu+d}}\underbrace{(2\nu)^{\nu+d/2}\left((2\pi)^{d/2}C_{d,\nu}\right)^{-1}C'(\mathcal{F}[f])}_{=:C^1_{d,\nu}(\mathcal{F}[f])}+\underbrace{\left((2\pi)^{d/2}C_{d,\nu}\right)^{-1}C'_{d,\nu}(\mathcal{F}[f])}_{=:C^2_{d,\nu}(\mathcal{F}[f])}\right) \\
        &=h^{2\nu}\left(\frac{1}{h^{2\nu+d}}C^1_{d,\nu}(\mathcal{F}[f])+C^2_{d,\nu}(\mathcal{F}[f])\right)
    \end{align*}
    which concludes the proof.
\end{proof}

\section{Derivation of $h^*$ in Eq. (\ref{eq:optimal_h})}
\label{ap:sec:derivation_h_opt}
Taking derivative of the bound presented in Proposition \ref{lemma:bound_norm} w.r.t. $h$ yields
\begin{equation*}
    \frac{\mathrm{d}}{\mathrm{d}h}\left(h^{2\nu}\left(\frac{1}{h^{2\nu+d}}C^1_{d,\nu}(\mathcal{F}[f])+C^2_{d,\nu}(\mathcal{F}[f])\right)\right)=\frac{1}{h^{d+1}}\left(2\nu C^2_{d,\nu}(\mathcal{F}[f])h^{2\nu+d}-d C^1_{d,\nu}(\mathcal{F}[f])\right).
\end{equation*}
Since $d,h>0$, the leading factor $1/h^{d+1}$ never becomes zero, so we must have
\begin{equation*}
    2\nu C^2_{d,\nu}(\mathcal{F}[f])h^{2\nu+d}-d C^1_{d,\nu}(\mathcal{F}[f])=0\iff h^*=\left(\frac{d}{2\nu}\frac{C^1_{d,\nu}(\mathcal{F}[f])}{C^2_{d,\nu}(\mathcal{F}[f])}\right)^{1/(2\nu+d)}
\end{equation*}
which shows what we have stated in Eq. (\ref{eq:optimal_h}) of the main paper.

\section{Proof of Theorem \ref{thm:matern_siren_ntk}}
\label{ap:sec:proof_matern_siren_ntk}
The proof of Theorem \ref{thm:matern_siren_ntk} is based on Bochner's theorem (in harmonic analysis) which reads:

\begin{theorem}[Bochner]
A continuous function $k:\mathbb{R}^d\longrightarrow\mathbb{R}$ with $k(0)=1$ is positive definite if and only if there exists a finite positive Borel measure $\mu$ on $\mathbb{R}^d$ such that $$k(\tau)=\int_{\mathbb{R}^d}e^{i\omega^\top\tau}\mathrm{d}\mu(\omega).$$
\end{theorem}
In essence, Bochner's theorem states that $k$ and $\mu$ are \textit{Fourier duals}; given a stationary kernel $k$, one can obtain $\mu$ (better known as \textit{spectral density} if normalized) by applying the \textit{inverse} Fourier transform to $k$.
On the other hand, given a spectral density $\mu$, one can obtain the corresponding kernel function $k$ by applying the Fourier transform to $\mu$.
Based on Bochner's theorem, 
\begin{equation*}
    k(\tau)=k(x-y)=\int_{\mathbb{R}^d}e^{i\omega^\top(x-y)}\mathrm{d}\mu(\omega)=\E_{\omega\sim\mu}\left[\cos(\omega^\top(x-y))\right]
\end{equation*}
because $k$ is real and for \textit{even} spectral densities.
On the other hand, \cite{rahimi2007} showed 
\begin{align*}
    k(\tau)=k(x-y)&=\E_{\omega\sim\mu}[\cos(\omega^\top(x-y))] \\
    &=\E_{\omega\sim\mu}[\cos(\omega^\top(x-y))]+\underbrace{\E_{\omega\sim\mu,b\sim\mathcal{U}(0,2\pi)}[\cos(\omega^\top(x+y)+2b)]}_{=0\,(*)} \\
    &=\E_{\omega\sim\mu,b\sim\mathcal{U}(0,2\pi)}[\cos(\omega^\top(x-y))+\cos(\omega^\top(x+y)+2b)] \\
    &=\E_{\omega\sim\mu,b\sim\mathcal{U}(0,2\pi)}[2\cos(\omega^\top x+b)\cos(\omega^\top y+b)],
\end{align*}
where $(*)$ follows because $b$ is uniformly distributed in $[0,2\pi]$ (hence, the inner expectation w.r.t. $b$ becomes zero).

We will now prove Theorem \ref{thm:matern_siren_ntk}, copied below for convenience:

\thmMaternSirenNtk*

\begin{proof}
    Let $f:\R^d\longrightarrow\R$ be a two-layer fully-connected network with sine activation function, \textit{i.e.},
    \begin{equation*}
        f(x)=\sqrt{\frac{2}{m}}\sum_{i=1}^m v_i\sin(w_i^\top x+b_i)
    \end{equation*}
    with bottom layer weights $W=(w_1,w_2,\dots,w_m)\in\R^{d\times m}$, biases $b=(b_1,b_2,\dots,b_m)\in\R^m$, and top-layer weights $v=(v_1,v_2,\dots,v_m)\in\R^m$.
    Assume that $W$ and $b$ are \textit{fixed} (that is, we only allow the second layer to be trained), and initialized according to some distribution.
    Choose
    \begin{equation}
        w_i\sim p_\nu(\omega)=h^{-2\nu}C_{d,\nu}\left(\frac{2\nu}{h^2}+(2\pi\Vert\omega\Vert)^2\right)^{-(\nu+d/2)}
        \label{eq:spectral_density}
    \end{equation}
    for $h,\nu>0$ and $b_i\sim\mathcal{U}(0,2\pi)$ for all $i\in\{1,2,\dots,n\}$.
    Then, based on the fact that $\partial_{v_i}f(x)=\sqrt{2}\sin(\omega_i^\top x+b_i)/\sqrt{m}$, it is easy to see that the NTK of $f$ is given by
    \begin{align*}
        k_{\text{NTK}}(x,y)&=\sum_{i=1}^m\partial_{v_i}f(x)\partial_{v_i}f(y) \\
        &=\frac{1}{m}\sum_{i=1}^m 2\sin(w_i^\top x+b_i)\sin(w_i^\top y+b_i) \\
        &=\frac{1}{m}\sum_{i=1}^m 2\cos(w_i^\top x+b_i)\cos(w_i^\top y+b_i)-2\cos(w_i^\top (x+y)+2b_i) \\
        &=\frac{1}{m}\sum_{i=1}^m 2\cos(w_i^\top x+b_i)\cos(w_i^\top y+b_i)-\frac{1}{m}\sum_{i=1}^m 2\cos(w_i^\top(x+y)+2b_i).
    \end{align*}
    As $m\rightarrow\infty$, we finally obtain
    \begin{align*}
        k_{\text{NTK}}(x,y)&=\E_{(w,b)}\left[2\cos(w^\top x+b)\cos(w^\top y+b)\right]-2\underbrace{\E_{(w,b)}\left[\cos(w^\top(x+y)+2b)\right]}_{=0\,(*)} \\
        \overset{(**)}&{=}k_\nu(x,y).
    \end{align*}
    Here, $(*)$ follows because $b$ is uniformly distributed, and $(**)$ since $w$ is distributed according to $p_\nu$ which is, in fact, the \textit{spectral density} of $k_\nu$, see, \textit{e.g.}, Eq. (10) in \cite{kanagawa2018}.
\end{proof}

In fact, Theorem \ref{thm:matern_siren_ntk} can be generalized in that the NTK of a \textsc{Siren} is generally any stationary kernel:

\begin{remark}
The NTK of a \textsc{Siren} can be associated with any \textit{stationary} kernel if the bottom layer weights in Eq. (\ref{eq:spectral_density}) are initialized according to the corresponding kernel's spectral density.
\end{remark}

This is a direct consequence of Bochner's theorem.

\section{Proof of Corollary \ref{cor:rkhs_matern_arc-cosine}}
\label{ap:sec:proof_rkhs_matern_arc-cosine}
We first re-state the following result, needed to prove the second part of Corollary \ref{cor:rkhs_matern_arc-cosine}:

\begin{theorem}[\cite{saitoh2016}, Theorem 2.17]\label{thm:pd_kernels_rkhs}
    Let $k_1$ and $k_2$ be two positive definite kernels. Denote by $H_1,H_2$ the RKHSs induced by $k_1$ and $k_2$. Then, $$H_1\subset H_2\iff\gamma^2k_2-k_1\text{ is positive definite for }\gamma>0.$$
\end{theorem}

We are now ready to prove Corollary \ref{cor:rkhs_matern_arc-cosine} which is copied below for convenience:

\corRkhsMaternArcCosine*
 
\begin{proof}
    The first part follows from Theorem~\ref{thm:rkhs_matern_laplace} by observing that the NTK of a fully-connected $(L=2$)-layer ReLU network (with \textit{all} parameters being trained) is
    \begin{equation}
        k_{\widetilde{\text{AC}}}(x,y)=\underbrace{k_\text{AC}(x,y)}_{\text{first-order}}+x^\top y\underbrace{\frac{1}{\pi}(\pi-\theta)}_{\text{zeroth-order}}
        \label{eq:full_relu_ntk}
    \end{equation}
    see, \textit{e.g.},~\cite{chen2021}. The second part follows from Theorem \ref{thm:rkhs_matern_laplace}, noting that $H_{\widetilde{\text{AC}}}\subset H_{\text{1/2}}$ when restricted to $\mathbb{S}^{d-1}$; so, based on Theorem \ref{thm:pd_kernels_rkhs}, $k:=\gamma^2 k_{\text{1/2}}-k_{\widetilde{\text{AC}}}$ is a valid kernel for $\gamma>0$. Re-arranging the equation, we have
    \begin{equation}
        k_{\text{1/2}}=a k_{\widetilde{\text{AC}}} +b,\quad\text{where}\quad a=1/\gamma^2,b=k/\gamma^2.
        \label{eq:laplace_arc-cosine-like}
    \end{equation}
    Now, denote by $k_\text{AC}^{(0)}(x,y)$ the zeroth-order arc-cosine kernel as defined in Eq. (\ref{eq:full_relu_ntk}), and further set $\hat{k}_\text{AC}^{(0)}(x,y)=x^\top yk_\text{AC}^{(0)}(x,y)$. The latter is a product between two kernel functions, hence a valid, \textit{i.e.}, positive definite, kernel itself. Plugin Eq. (\ref{eq:full_relu_ntk}) into Eq. (\ref{eq:laplace_arc-cosine-like}) yields
    \begin{equation}
        \begin{split}
            k_{\text{1/2}}&=a(k_\text{AC}+\hat{k}_\text{AC}^{(0)})+b\\
            &=ak_\text{AC}+c
        \end{split}\quad\text{with}\quad c=a\hat{k}_\text{AC}^{(0)}+b,
    \end{equation}
    where $c$ is again a valid kernel function since it is a sum of two kernels.
    This shows that the Laplace and first-order arc-cosine kernel are indeed related by a constant scaling factor $a$ and the addition of a kernel function $c$ when restricted to $\mathbb{S}^{d-1}$, concluding the proof.
\end{proof}

In other words, Corollary \ref{cor:rkhs_matern_arc-cosine} states that, when restricted to the sphere, the Laplace kernel is a \textit{linear combination of two kernel functions}: the first-order arc-cosine kernel and another, valid kernel function (that accounts for the residuals).
This holds also in the other direction.

\section{Additional Details for ShapeNet Experiment}
\label{ap:sec:shapenet_non_learnable_details}
In this section, we provide implementation details and additional qualitative results on the ShapeNet experiment as presented in Section \ref{subsec:results_shapenet} of the main paper.

\subsection{Implementation Details}
\label{ap:subsec:impl_shapenet}
Following common practice (see, \textit{e.g.}, \cite{williams2022}), we use \textit{finite differences} to approximate the gradient part of the KRR problem in Eq. (\ref{eq:emp_risk}) of the main paper. 
Specifically, denote $\mathcal{X}':=\mathcal{X}^+\cup\mathcal{X}^-$, where $\mathcal{X}^+:=\{x_1+\eps n_1,x_2+\eps n_2,\dots,x_m+\eps n_m\}$ and $\mathcal{X}^-:=\{x_1-\eps n_1,x_2-\eps n_2,\dots,x_m-\eps n_m\}$ for a fixed $\eps>0$.
The Representer Theorem \citep{kimeldorf1970,schoelkopf2001} tells us that the solution to Eq. (\ref{eq:emp_risk}) is of the form 
\begin{equation*}
    f(x)=\sum_{i=1}^{2m}\alpha_i k(x,x_i')
\end{equation*}
which is linear in the coefficients $\alpha\in\mathbb{R}^{2m}$, given by
\begin{equation*}
    \alpha=(K+\lambda I)^{-1}y,\quad\text{where}\quad K=(K)_{ij}=k(x_i,x_j)\in\mathbb{R}^{2m\times 2m},
\end{equation*}
and
\begin{equation*}
    y_i=\begin{cases} 
        +\eps & \text{if }x'_i\in\mathcal{X}^+,\\
        -\eps & \text{if }x'_i\in\mathcal{X}^-,
    \end{cases}\quad\text{for }i=1,2,\dots,2m.
\end{equation*}
Here, $I\in\mathbb{R}^{2m\times 2m}$ denotes the identity matrix.
We employ PyTorch's built-in Cholesky solver to numerically stable solve for $\alpha$ and set $\eps=0.005$ for all experiments.

\subsection{Further Qualitative Results}
We show additional qualitative results on ShapeNet in Figure \ref{fig:ap_non_learnable}.

\begin{figure}
    \centering
    \includegraphics[width=\linewidth]{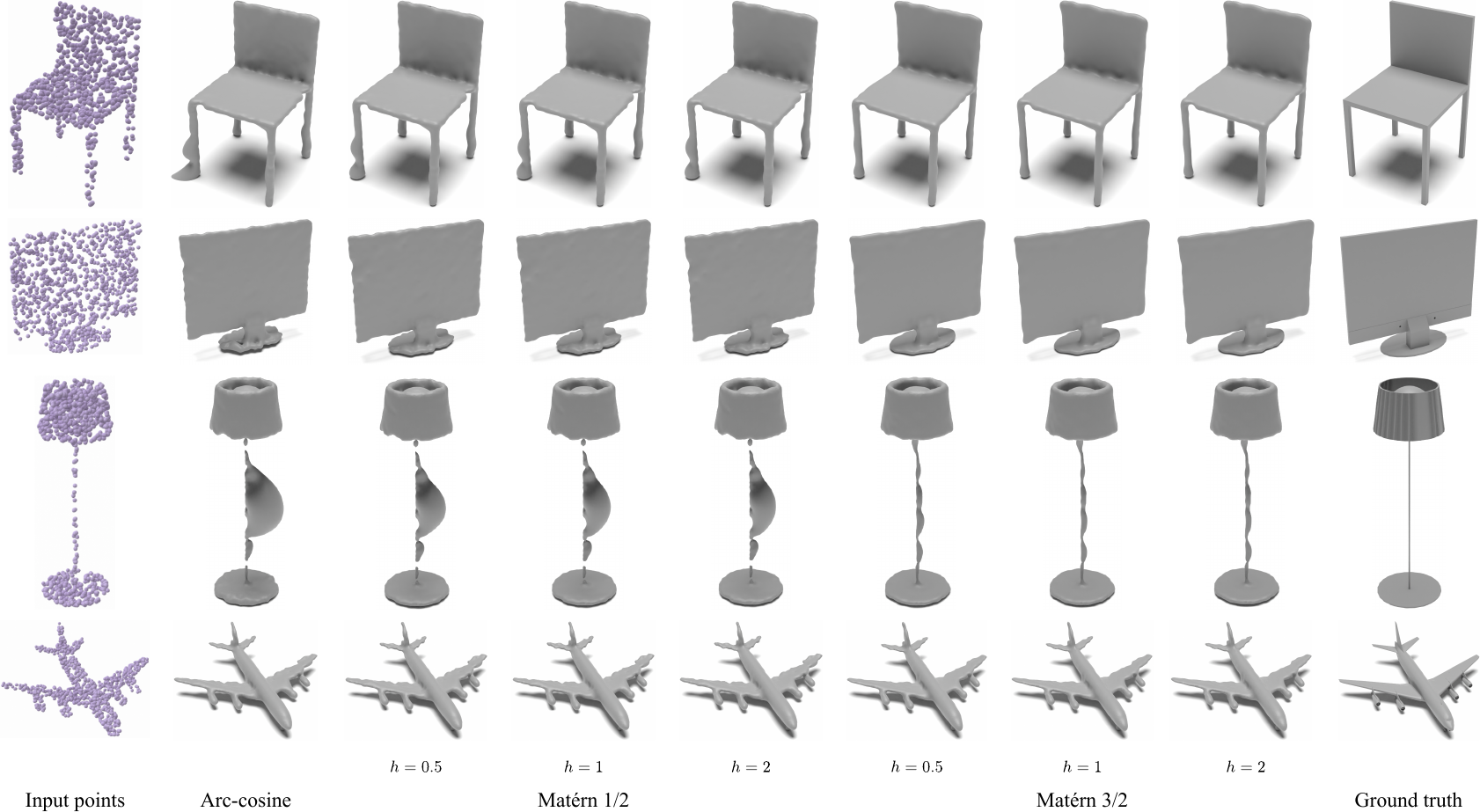}
    \caption{\textbf{Additional qualitative results on ShapeNet for non-learnable kernels}. We compare the arc-cosine kernel against \Matern 1/2 and 3/2 for different values of $h$.}
    \label{fig:ap_non_learnable}
\end{figure}

\section{Further Details on the Surface Reconstruction Benchmark}
\label{ap:sec:details_srb}
This section details the experimental setting that has been used to run \textit{Neural Splines} (NS; \citet{williams2021}) on SRB in Section \ref{subsec:results_shapenet} of the main paper.
Moreover, we present per-object metrics and additional qualitative results. 

\subsection{Experimental Setting}
\label{subsec:exp_srb}
For NS, we used the official implementation provided here: \url{https://github.com/fwilliams/neural-splines}.
Similar to the original paper, we used 15,000 Nyström samples and did the same parameter sweep over the regularization parameter, $\lambda\in\{0,10^{-13},10^{-12},10^{-11},10^{-10}\}$. 
For the rest of the parameters (not mentioned in the paper), default values provided in the repository have been used, except for the grid size which we set to 512.
We used exactly the same setting for \Matern kernels, except that we did a modest parameter sweep over the bandwidth parameter, $h\in\{0.5,1,2\}$.
We utilized the SRB data from here: \url{https://github.com/fwilliams/deep-geometric-prior}.

Despite our best efforts, we were unable to reproduce the results for the NS kernel reported in the original paper.
We run on four different GPUs, including NVIDIA RTX A2000, A40, A100, and V100 (which has been used by the authors of Neural Splines), and tested different configurations of the following parameters (see code): \texttt{grid\_size}, \texttt{eps}, \texttt{cg-max-iters}, \texttt{cg-stop-thresh}, \texttt{outer-layer-variance}.
In Table \ref{tab:results_non_learnable} of the main paper, we report the best result we could achieve on a V100, which is 0.19 for Chamfer and 3.19 for the Hausdorff distance; \cite{williams2021} reported 0.17 and 2.85 for Chamfer and Hausdorff distance, respectively. 
In order to achieve this, we had to lower \texttt{outer-layer-variance} form $10^{-3}$ per default to $10^{-5}$.

To measure runtime, we again used the same setting as in the original NS paper: 15,000 Nyström samples, and $\lambda=10^{-11}$.
Due to the absence of further information about the experimental setting in the original paper, we used default values provided in the repository for the rest of the parameters.
Runtime was measured on a single NVIDIA V100, similar to \cite{williams2021}.

\subsection{Additional Results}
Per-object metrics are shown in Table \ref{tab:per-object-srb}, and more qualitative results in Figure \ref{fig:srb-qualitative}.

\begin{table}
\begin{tabular}{lcc|cc|cc|cc|cc}
\toprule
& \multicolumn{2}{c|}{Anchor} & \multicolumn{2}{c|}{Daratech} & \multicolumn{2}{c|}{DC} & \multicolumn{2}{c|}{Gargoyle} & \multicolumn{2}{c}{Lord Quas} \\
\cmidrule(lr){2-11}
& CD $\downarrow$ & HD $\downarrow$ & CD $\downarrow$ & HD $\downarrow$ & CD $\downarrow$ & HD $\downarrow$ & CD $\downarrow$ & HD $\downarrow$ & CD $\downarrow$ & HD $\downarrow$ \\
\midrule
\Matern 1/2 & 0.31 & 6.97 & 0.25 & 6.11 & 0.18 & 2.62 & 0.19 & 5.14 & \textbf{0.12} & 1.32 \\
\Matern 3/2 & 0.25 & 5.08 & 0.23 & 4.90 & \underline{0.15} & \textbf{1.24} & \textbf{0.16} & \textbf{2.54} & \textbf{0.12} & 0.90 \\
\Matern 5/2 & 0.93 & 28.33 & 0.87 & 32.17 & 0.34 & 15.98 & 0.55 & 22.52 & 0.21 & 9.49 \\
\Matern $\infty$ & 2.88 & 29.48 & 5.80 & 45.94 & 0.35 & 35.70 & 3.84 & 30.70 & 3.13 & 24.55 \\
\midrule
NS$^*$ & 0.27 & 5.38 & 0.23 & 4.67 & \underline{0.15} & 1.41 & \underline{0.17} & 3.49 & \textbf{0.12} & 0.99 \\
\midrule
NS & \underline{0.22} & 4.65 & \textbf{0.21} & 4.35 & \textbf{0.14} & \underline{1.35} & \textbf{0.16} & \underline{3.20} & \textbf{0.12} & \textbf{0.69} \\
SPSR & 0.33 & 7.62 & 0.26 & 6.62 & 0.17 & 2.79 & 0.18 & 4.60 & \textbf{0.12} & 1.83 \\
FFN & 0.31 & \underline{4.49} & 0.34 & 5.97 & 0.20 & 2.87 & 0.22 & 5.04 & 0.35 & 3.90 \\
\textsc{Siren} & 0.32 & 8.19 & \textbf{0.21} & 4.30 & \underline{0.15} & 2.18 & \underline{0.17} & 4.64 & 0.17 & \underline{0.82} \\
SAP & 0.34 & 8.83 & \underline{0.22} & \underline{3.09} & 0.17 & 3.30 & 0.18 & 5.54 & \underline{0.13} & 3.49 \\
DiGS & 0.28 & 5.71 & \textbf{0.21} & 5.02 & \underline{0.15} & 2.13 & \textbf{0.16} & 3.81 & \textbf{0.12} & 1.10 \\
VisCo & \textbf{0.21} & \textbf{3.00} & 0.26 & 4.06 & \underline{0.15} & 2.22 & \underline{0.17} & 4.40 & \textbf{0.12} & 1.06 \\
OG-INR & 0.29 & 7.56 & 0.23 & \textbf{2.89} & 0.17 & 2.68 & 0.19 & 5.01 & \underline{0.13} & 2.14 \\
\bottomrule
\end{tabular}
\caption{\textbf{Per-object quantitative results on SRB}. "NS*" denotes the best result we could achieve by running the official implementation of NS, see Section \ref{subsec:exp_srb}. \textbf{Bold} marks best result, \underline{underline} second best.}
\label{tab:per-object-srb}
\end{table}

\begin{figure}
    \centering
    \includegraphics[width=0.7\linewidth]{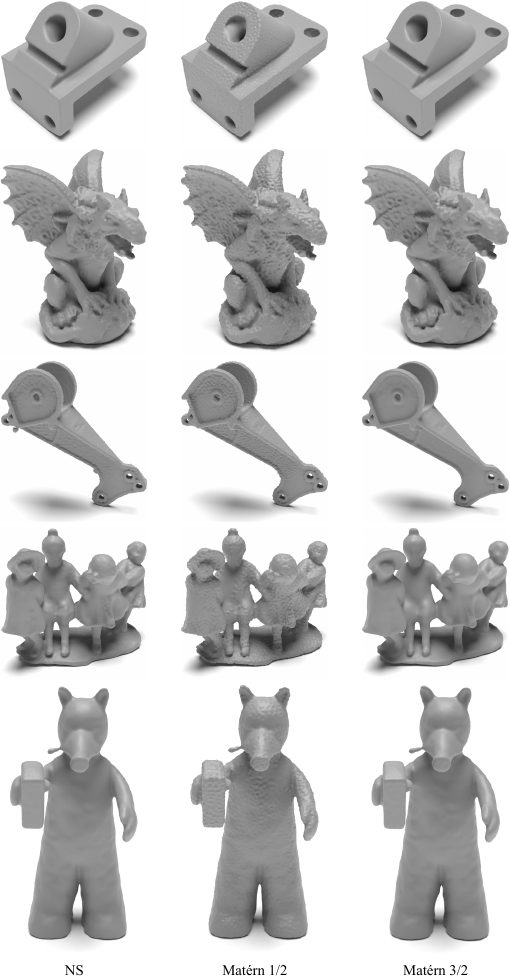}
    \caption{\textbf{Additional qualitative results on SRB for non-learnable kernels}. We compare \textit{Neural Splines} (NS; \citet{williams2021}) to \Matern 1/2 and 3/2.}
    \label{fig:srb-qualitative}
\end{figure}

\section{Additional Details For Texture Reconstruction Experiment}
\label{ap:sec:texture_recon}
This section provides implementation details and further information about the evaluation protocol used for the texture reconstruction experiment presented in Section \ref{subsec:texture_recon} of the main paper.

\subsection{Implementation Details}
Recall from Section \ref{subsec:texture_recon} of the main paper that we are given a dataset $\mathcal{D}'=\mathcal{D}\times\{c_1,c_2,\dots,c_m\}\subset\Omega\times\mathbb{R}^d\times\mathbb{R}^3$ with per-point RGB colors $\{c_1,c_2,\dots,c_m\}$.
In order to also reconstruct texture, we are seeking a function $f':\mathbb{R}^d\rightarrow\mathbb{R}^4$, predicting signed distances as well as RGB values.
Using notation introduced in Section \ref{ap:subsec:impl_shapenet}, we compute $\alpha'\in\mathbb{R}^{2m\times 4}$ as
\begin{equation*}
    \alpha'=(K+\lambda I)^{-1}y',\quad\text{where}\quad y':=[y,\Bar{c}]\in\mathbb{R}^{2m\times 4}
\end{equation*}
and $\Bar{c}:=[c, c]^\top$ with $(c_1,c_2,\dots,c_m)\in\mathbb{R}^{m\times 3}$.
Then, we predict the function $f'$ as 
\begin{equation*}
    f'(x)=\sum_{i=1}^{2m}\alpha_i' k(x,x_i'),
\end{equation*}
where $\alpha_i'\in\mathbb{R}^4$ denotes the $i$-th row of $\alpha'$.
Clearly, $f'(x)=(f_1'(x),f_2'(x),f_3'(x),f_4'(x))\in\mathbb{R}^4$ with $f'_1$ denoting the singed distance at $x$, and $f'_2,f'_3,f'_4$ are the component functions that represent RGB colors at a position $x$.

To extract a textured surface mesh from the predicted four-dimensional volume, we use a two-stage process. 
First, we apply Marching Cubes \citep{lorensen1987} to $f'_1$, and then trilinearily interpolate RGB colors at the previously predicted surface points.

\subsection{Evaluation Protocol}
Our quantitative evaluation on the \textit{Google Scanned Objects} (GSO; \cite{downs2022}) dataset closely follows \cite{mitchel2024}, in which authors compute image-based evaluation metrics on the texture atlas to quantify reconstruction ability.
Specifically, we first sample $m$ surface points along with corresponding normals and per-point RGB colors from shapes in GSO. 
Next, we solve the KRR problem in Eq. (\ref{eq:emp_risk}) as always. 
Then, instead of predicting RGB colors on the estimated geometry (which generally differs between various kernel functions), we predict (\textit{i.e.}, trilinearily interpolate) colors on the surface points of the \textit{original} mesh. 
This allows us to leverage the existing UV parametrization of the original mesh to generate a texture atlas with colors predicted from KRR.
Once the texture atlas at hand, we compare it to the original texture atlas using PSNR and LPIPS.

We used ten objects from GSO with different complexity.
The following assets have been used:
\begin{itemize}
    \item \texttt{ACE\_Coffee\_Mug\_Kristen\_16\_oz\_cup}
    \item \texttt{Animal\_Planet\_Foam\_2Headed\_Dragon}
    \item \texttt{Jansport\_School\_Backpack\_Blue\_Streak}
    \item \texttt{Now\_Designs\_Bowl\_Akita\_Black}
    \item \texttt{Predito\_LZ\_TRX\_FG\_W}
    \item \texttt{Reebok\_ZIGTECH\_SHARK\_MAYHEM360}
    \item \texttt{Razer\_Kraken\_Pro\_headset\_Full\_size\_Black}
    \item \texttt{RJ\_Rabbit\_Easter\_Basket\_Blue}
    \item \texttt{Schleich\_Lion\_Action\_Figure}
    \item \texttt{Weisshai\_Great\_White\_Shark}
\end{itemize}
They can be downloaded from: \url{https://app.gazebosim.org/dashboard}.

\section{Additional Results for Data-dependent Kernels}
\label{ap:sec:shapenet_learnable_additional_results}
We show additional qualitative results for learned kernels on ShapeNet in Figure \ref{fig:ap_learnable}.

\begin{figure}
    \centering
    \includegraphics[width=\linewidth]{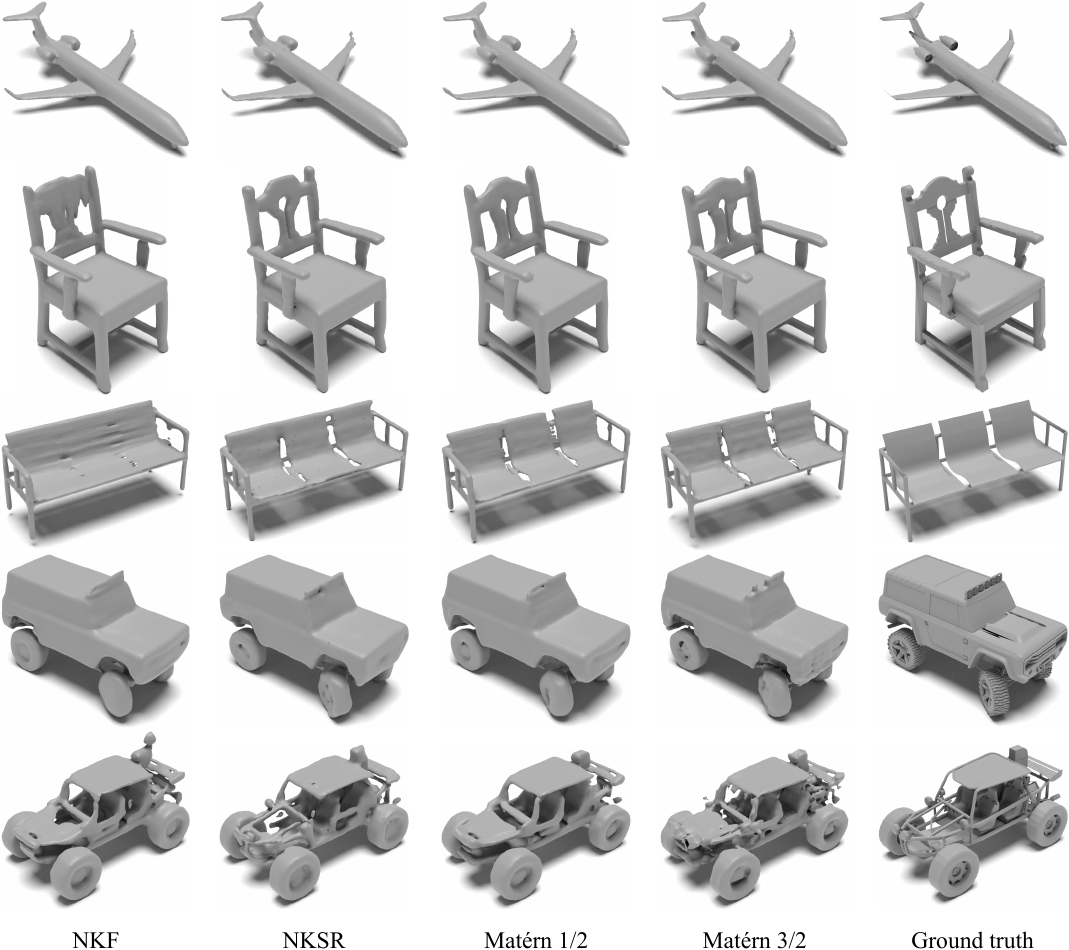}
    \caption{\textbf{Additional qualitative results on ShapeNet for learned kernels}. We compare \textit{Neural Kernel Fields} (NKFs; \citet{williams2022}) and \textit{Neural Kernel Surface Reconstruction} (NKSR; \citet{huang2023}) against learnable \Matern 1/2 and 3/2.}
    \label{fig:ap_learnable}
\end{figure}

\end{document}

%% file: math_commands.tex
\usepackage{amsmath,amsfonts,bm}

\def\eqref#1{equation~\ref{#1}}

\def\1{\bm{1}}

\def\eps{{\epsilon}}

\DeclareMathAlphabet{\mathsfit}{\encodingdefault}{\sfdefault}{m}{sl}
\SetMathAlphabet{\mathsfit}{bold}{\encodingdefault}{\sfdefault}{bx}{n}

\newcommand{\E}{\mathbb{E}}

\newcommand{\R}{\mathbb{R}}

\DeclareMathOperator*{\argmin}{arg\,min}